\documentclass[a4paper,11pt]{article}
\pdfoutput=1

\usepackage[utf8]{inputenc}
\usepackage[british]{babel}
\usepackage[mono=false]{libertine}
\usepackage[T1]{fontenc}
\usepackage[top=2.5cm, bottom=2.5cm, left=2.5cm, right=2.5cm]{geometry}
\usepackage{titlesec}
\usepackage[titletoc]{appendix}
\usepackage{amsmath}
\usepackage{amsthm}
\usepackage{amssymb}
\usepackage{url}
\usepackage{natbib}
\usepackage{times}
\usepackage{enumitem}
\usepackage[usenames, dvipsnames]{color}
\usepackage{titlesec}
\usepackage{array}
\usepackage{mathtools}
\usepackage{amsfonts}
\usepackage{dsfont}
\usepackage{bm}
\usepackage{cases}
\usepackage{setspace}
\usepackage[perpage]{footmisc}

\usepackage{booktabs}       
\usepackage{nicefrac}       
\usepackage{microtype}      
\usepackage{hyperref}
\usepackage{natbib}
\hypersetup{colorlinks=true,urlcolor=blue,citecolor=magenta,linkcolor=blue}

\newcommand{\changelocaltocdepth}[1]{%
  \addtocontents{toc}{\protect\setcounter{tocdepth}{#1}}%
  \setcounter{tocdepth}{#1}%
}

\DeclareUnicodeCharacter{00A0}{~}

\def \({\left(}
\def \){\right)}
\def \[{\left[}
\def \]{\right]}

\newcommand{\bq}{{\textbf {q}}}
\newcommand{\bb}{{\textbf {b}}}

\newcommand{\bu}{{\textbf {u}}}

\newcommand{\bY}{{\textbf {Y}}}

\newcommand{\bbe}{{\textbf {e}}}

\newcommand{\bZ}{{\textbf {Z}}}

\newcommand{\bw}{{\textbf {w}}}
\newcommand{\bv}{{\textbf {v}}}

\newcommand{\bX}{{\textbf {X}}}
\newcommand{\bx}{{\textbf {x}}}

\newcommand{\bt}{{\textbf {t}}}

\newcommand{\by}{{\textbf {y}}}
\newcommand{\bz}{{\textbf {z}}}

\newcommand{\bxi}{{\boldsymbol{\xi}}}

\newcommand{\ba}{{\textbf {a}}}

\newcommand{\be}{\begin{equation}}
\newcommand{\ee}{\end{equation}}

\newcommand\smallO{
  \mathchoice
    {{\scriptstyle\mathcal{O}}}
    {{\scriptstyle\mathcal{O}}}
    {{\scriptscriptstyle\mathcal{O}}}
    {\scalebox{.7}{$\scriptscriptstyle\mathcal{O}$}}
  }
\newcommand{\bea}{\begin{align}}
\newcommand{\eea}{\end{align}}
\newcommand{\norm}[1]{\left\lVert#1\right\rVert}

\newenvironment{proofof}[1]{\textbf{{\em Proof of #1.}}}{\hfill \rule{2.5mm}{2.5mm} 
 }

\newtheorem{theorem}{Theorem}
\newtheorem{defn}{Definition}
\newtheorem{result}{Result}
\newtheorem{assumption}{Assumption}
\newtheorem{lemma}[theorem]{Lemma}
\newtheorem{proposition}[theorem]{Proposition}

\newtheorem{model}{\textbf{Model}}

\DeclareMathAlphabet{\varmathbb}{U}{bbold}{m}{n}

\newcommand{\EE}{\mathbb{E}}
\newcommand{\bbR}{\mathbb{R}}
\newcommand{\bbP}{\mathbb{P}}

\newcommand{\bbN}{\mathbb{N}}

\newcommand{\bbC}{\mathbb{C}}

\DeclareMathOperator*{\extr}{\mathrm{extr}}

\showboxdepth=\maxdimen
\showboxbreadth=\maxdimen

\begin{document}
\setcounter{tocdepth}{2}

\title{Landscape Complexity for the Empirical Risk of Generalized Linear Models}
\author{Antoine Maillard$^{\star,\otimes}$, G\'erard Ben Arous$^\dagger$, Giulio Biroli$^\star$}
\date{}
\maketitle
{\let\thefootnote\relax\footnote{
\!\!\!\!\!\!\!\!\!\!
$\star$ Laboratoire de Physique de l’Ecole Normale Sup\'erieure, ENS, Universit\'e PSL, CNRS, Sorbonne Universit\'e, Universit\'e Paris-Diderot, Sorbonne Paris Cit\'e, Paris, France. \\
$\dagger$ Courant Institute of Mathematical Sciences,  New York University,  251 Mercer Street,  New York,  NY 10012.\\
$\otimes$ To whom correspondence shall be sent: \href{mailto:antoine.maillard@ens.fr}{antoine.maillard@ens.fr}
}}
\setcounter{footnote}{0}

\begin{abstract}%
  We present a method to obtain the average and the typical value of the 
  number of critical points of the empirical risk landscape for generalized linear estimation problems and variants. 
  This represents a substantial extension of previous applications of the Kac-Rice method since it allows to analyze the critical points of high dimensional non-Gaussian random functions. 
  Under a technical hypothesis, we obtain a rigorous explicit variational formula for the \emph{annealed complexity}, which is the logarithm of the average number of critical points at fixed value of the empirical risk. This result is simplified, and extended, using the non-rigorous Kac-Rice replicated method from theoretical physics. In this way we find an explicit variational formula for the \emph{quenched complexity}, which is generally different from its annealed counterpart, and allows to obtain the number of critical points for typical instances up to exponential accuracy. 
\end{abstract}

\section{Introduction and main results}\label{sec:introduction}

\subsection{Introduction}

Characterizing the landscape of the empirical risk
is a key issue in several contexts. Many current 
machine learning problems are both non-convex and high-dimensional. 
In these cases, the analysis of optimization algorithms, such as gradient 
descent and its stochastic variants, represents a very hard challenge. In recent years, there has been a series of works that developed a landscape-based approach to tackle this challenge. The key idea is to study the statistical properties of the empirical risk landscape, and to use these findings to obtain results on the performance of algorithms. Without the 
aim of being exhaustive this research avenue includes analysis of the landscape of 
 neural networks, matrix completion, tensor factorization and tensor principal component analysis \citep{fyodorov2004complexity,fyodorov2012critical,kawaguchi2016deep,soudry2016no,ge2016matrix,1611.01540,bhojanapalli2016global,park2016non,du2017gradient,ge2017optimization,ge2017no,lu2017depth,ling2018landscape,1711.05424,ros2019complex,mannelli2019passed,mannelli2019afraid,arous2018algorithmic,biroli2019iron}.
The majority of these works identifies the region of parameters  where the landscape is ``easy'', i.e.\ it focuses on the regime where there shouldn't be any bad local minima and it proves that indeed there are none. However, gradient descent and other landscape-based algorithms are often observed to work even very far from the region described above where the landscape can be proved mathematically to be ``easy''. A possible reason is that the bounds obtained rigorously are not tight enough. Another, more interesting, is that the landscape is ``hard'', i.e.\ spurious minima are present, but their basins of attraction are small and the dynamics is able to avoid them \citep{mannelli2019passed,mannelli2019afraid}.  

Here we develop a general method that allows to study and focus directly on the ``hard'' regime, where the empirical risk 
displays a huge number of bad minima. Our aim is to obtain explicit formulas for the number of critical points of the empirical risk landscape, to characterize their indices, and the Hessian associated to them. For a given problem, this will allow to identify the topological transition where the landscape becomes ``easy'', and to analyze very precisely the ``hard'' regime.
In recent years, there has been remarkable progress 
on this subject in the field of spin-glasses and probability theory through the Kac-Rice method \citep{Fyodorov2007,PhysRevLett.98.150201,1110.5872,1003.1129,subag2017complexity,1711.05424,ros2019complex}. This line of research has allowed to put on a firm ground results previously obtained in the physics literature
\citep{0022-3719-13-31-006,kurchan1991replica,crisanti1995thouless,0305-4470-32-5-004}, and it has unveiled important relationships with random matrix theory. Its main domain of application has been the study of the landscapes associated to Gaussian random functions. 
Its extension to tackle the case of non-Gaussian high-dimensional random functions is an open problem---one that is crucial to address in order to characterize the critical points of the empirical risk.

Here we present an important step forward: an extension of the Kac-Rice method to compute the number of critical points of the empirical risk arising in generalized linear estimation problems \citep{nelder1972generalized,mccullagh,barbier2017phase}. Our approach contains both mathematically rigorous analysis and exact results obtained by theoretical physics methods. 
We work out a rigorous formula (under a technical assumption) for the logarithm of the average of the number of the critical points, called henceforth \emph{annealed complexity}. This quantity already provides interesting information in itself. A more refined quantity, but much more challenging to be analyzed rigorously, is the average of the logarithm of the number of critical points. 
This so-called \emph{quenched complexity} is truly representative of the typical properties of the landscape for a given instance of the empirical risk and is generically different from its annealed counterpart which is instead dominated by rare instances, see for example 
the case of tensor PCA \citep{ros2019complex}. 
In order to obtain the quenched complexity, we develop a non-rigorous but exact approach that combines the Kac-Rice method with the replica theory used by physicists \citep{ros2019complex}. 

Our main results are two explicit formulas for the quenched and annealed complexities. They open the way towards a full fledged characterization of the landscape of generalized estimation problems and variants, and the analysis of landscape-based algorithms, such as gradient descent. Note that these models 
can also be viewed as the simplest neural network (single-node) in a teacher-student setting \citep{engel2001statistical}. 

\subsection{Main results}
We consider two classes of high-dimensional random functions. The first one 
is a kind of energy that arises in a simple model of neural networks (the perceptron, cf.\ \cite{engel2001statistical}) and in mean-field glass models \citep{franz2016simplest}:
\begin{align}\label{eq:model_nosignal}
    L_1(\bx) &\equiv \frac{1}{m} \sum_{\mu=1}^m  \phi(\bxi_\mu \cdot \bx),
\end{align}
where $\phi : \bbR \to \bbR$ is a smooth activation function (the hypotheses on $\phi$ are precised later), $\bx \in \mathbb{S}^{n-1}$, the unit sphere in $n$ dimensions, and $\bxi_\mu$ are i.i.d.\ random variables generated from the standard Gaussian distribution in $\bbR^n$.
The second class of functions we will consider are related to the loss functions of generalized linear models (GLMs) \cite{nelder1972generalized,mccullagh}.
In generalized linear estimation an observer has to infer a hidden vector $\bx^\star \in \mathbb{S}^{n-1}$ from the observation of the $m$-dimensional output vector $\bY = \{\phi(\bxi_\mu \cdot \bx^\star)\}_{\mu=1}^m$.
In this sense, the GLMs generalize the usual linear regression by allowing the output function to be non-linear\footnote{One can also consider GLMs in which the output function is stochastic. Here, we restrict to deterministic outputs.}.
We consider here \emph{random} GLMs, meaning that the data (or measurement) matrix $\bxi$ is taken random, with an i.i.d.\ standard Gaussian distribution, and we assume that the function $\phi$ and the data matrix 
$\bxi$ are given to the observer. This naturally leads to the mean square loss $L_2$:
\begin{align}\label{eq:model_signal}
    L_2(\bx) &\equiv \frac{1}{2m} \sum_{\mu=1}^m \left[\phi(\bxi_\mu \cdot \bx^\star) - \phi(\bxi_\mu \cdot \bx)\right]^2,
\end{align}
GLMs (and their random versions) arise in many different areas of
statistics, such as e.g.\ compressed sensing, phase retrieval, logistic regression, or in random artificial neural networks; we refer to \cite{barbier2017phase}
for a review of its numerous applications. 

Here, we are interested in the statistics of the number of critical points, or more precisely the complexity of the associated empirical risk (\ref{eq:model_signal}). 
For any open intervals $B\subseteq \ensuremath{\mathbb{R}_+}$ and $Q \subseteq (-1,1)$, we consider the (random) number
$\mathrm{Crit}_{n,L_2}(B,Q)$ of critical points of the function $L_2$ with loss value in $B$ and overlap  with the signal $q \equiv \bx \cdot \bx^\star$ in $Q$:
\begin{align}
      \label{defWk}
  \mathrm{Crit}_{n,L_2}(B,Q)\equiv 
  \sum_{\bx: \mathrm{grad} L_2(\bx) = 0 } 
  \mathds{1}\{ L_2(\bx) \in B, \ \bx \cdot \bx^\star \in Q\} .
\end{align}
Here $\mathrm{grad}$ is the Riemannian gradient on $\mathbb{S}^{n-1}$. For $L_1$ we define the similar quantity $\mathrm{Crit}_{n,L_1}(B)$, dropping the notion of overlap. 
The average $\EE \,\mathrm{Crit}_{n,L_2}(B,Q)$ is the quantity that can be analyzed rigorously. Its logarithm divided by $n$ is called the 
\emph{annealed complexity}. 
However, since the random variable $\mathrm{Crit}_{n,L_2}(B,Q)$ is in general strongly fluctuating and scales 
exponentially with $n$, its typical value is different from the mean and can be obtained by taking the exponential of $\EE \ln \mathrm{Crit}_{n,L_2}(B,Q)$. This last quantity (divided by $n$) is called the \emph{quenched complexity}. It is in general different from the annealed one, with very few exceptions \citep{subag2017complexity,crisanti1995thouless}. 

Our main results consist 
in explicit formulas for the annealed and quenched complexities for $L_1$ and $L_2$. The formula for the annealed case is obtained by a rigorous Kac-Rice method, whereas the one for the quenched complexity is obtained by theoretical physics methods combining the Kac-Rice method with replica theory.
We consider the limit $n,m \to \infty$ with $m/n \to \alpha > 1$, a setting called in the statistical physics literature the thermodynamic limit. The condition $\alpha > 1$ is essential, as can be seen for instance in eq.~\eqref{eq:model_nosignal} : if $m < n$, for each realization of $\{\bxi_\mu\}$, the function $L_1$ has an infinite number of critical points in the set of unit-norm $\bx$ orthogonal to all the $\{\bxi_\mu\}$, and counting the critical points in this case is meaningless (or one has to quotient the space to take care of the degeneracy). Our results hold for many classical activation functions $\phi$, such as e.g.\ the hyperbolic tangent, the arctangent, the sigmoid, or a smoothed and leaky version of the ReLU activation function\footnote{The precise hypotheses on the activation function $\phi$ are precised in Section~\ref{sec:annealed}.}.
Henceforth we shall denote $\bbC_+$ the strict upper-half complex plane, and ${\cal M}(\bbR^k)$ the set of probability measures on $\bbR^k$. 
For two probability measures $\mu$ and $\nu$ we define the relative entropy $H(\mu|\nu) \equiv \int \ln (\mathrm{d} \mu / \mathrm{d} \nu) \mathrm{d} \mu$ if $\mu$ is absolutely continuous with respect to $\nu$, and $+\infty$ otherwise. 
Finally, $\mu_G$ is a generic notation for the standard Gaussian measure on any $\bbR^k$. 

We can now present our main results. 
\begin{theorem}[The annealed complexity of $L_1$]\label{thm:annealed_nosignal} Let $B \subseteq \bbR$ a non-empty open interval and denote ${\cal M}_\phi(B)$ the set of probability measures $\nu$ on $\bbR$ such that $\int \nu(\mathrm{d}t) \phi(t)\in B$. 
Assume that the technical Assumption~\ref{hyp:technical} holds.
Given:
\begin{itemize}
\item $\mathcal{E}_\phi(\nu) \equiv \ln \left[\int \nu(\mathrm{d}x) \phi'(x)^2\right]$,
\item $t_\phi(\nu) \equiv \int \nu(\mathrm{d}x) x \phi'(x)$,
\item Let $\bz \in \bbR^{n \times m}$ an i.i.d.\ standard Gaussian matrix, and $\by \in \bbR^m$ a vector with components taken i.i.d.\ from a probability measure $\nu$. 
Let $D^{(\nu)}$ the diagonal matrix of size $m$ with elements $D^{(\nu)}_{\mu} = \phi''(y_\mu)$.
We define $\mu_{\alpha,\phi}[\nu]$ as the asymptotic spectral measure of $\bz D^{(\nu)} \bz^\intercal / m$.
\item $\kappa_{\alpha,\phi}(\nu,C) \equiv \int  \mu_{\alpha,\phi}[\nu](\mathrm{d}x) \ln \left|x - C\right|$,
\end{itemize}
one has\footnote{A fully rigorous statement would imply a lower and an upper bound given by a supremum over the adherence and the interior of ${\cal M}_\phi(B)$. For reasons of lightness and clarity of the presentation we write it in the simpler presented form.}:
\begin{align*}
\lim_{n \to \infty} \frac{1}{n}\ln \EE \ \mathrm{Crit}_{n,L_1}(B) &= \frac{1 + \ln \alpha}{2}  + \sup_{\nu \in {\cal M}_\phi(B)} \,  \left[ - \frac{1}{2}\mathcal{E}_\phi(\nu) + \kappa_{\alpha,\phi}(\nu,t_\phi(\nu)) - \alpha H(\nu|\mu_G)\right].
\end{align*}
\end{theorem}
\paragraph{A note on free probability}
Interestingly, the measure $\mu_{\alpha,\phi}[\nu]$ can be interpreted as the free multiplicative convolution of the Marchenko-Pastur law (at ratio $\alpha$) and the asymptotic spectral distribution of the matrix $D^{(\nu)}$, cf.\ e.g.\  \cite{voiculescu1987multiplication,anderson2010introduction}\footnote{Free multiplication is usually defined for positively-supported measures, however one can generalize it here by explicitly separating the positive and negative parts of $\phi''$ (we can show freeness of the 
 resulting two random matrices).}. We describe in Section~\ref{sec:variational} how to explicitly compute the density of $\mu_{\alpha,\phi}[\nu]$, 
 or its linear spectral statistics (as e.g.\ $\kappa_{\alpha,\phi}(\nu,C)$), via the computation of its Stieljtes transform. 

We turn to our second annealed result:
\begin{theorem}[The annealed complexity of $L_2$]\label{thm:annealed_signal}
Let $B \subseteq \bbR_+$ and $Q \subseteq (-1,1)$ two non-empty open intervals. For $q \in (-1,1)$ we denote ${\cal M}_\phi(B,q)$ the set of probability measures $\nu$ on $\bbR^2$ such that
\begin{align}
\begin{cases}
 &\int \nu(\mathrm{d}x,\mathrm{dy}) \, y \, \phi'(x) \left[\phi\left(q x+ \sqrt{1-q^2} y\right) - \phi(x)\right] = 0 , \\
&\int \nu(\mathrm{d}x,\mathrm{dy}) \left[\phi\left(q x+ \sqrt{1-q^2} y\right) - \phi(x)\right]^2 \in B.
\end{cases}
\end{align}
Given:
\begin{itemize}
\item $\mathcal{E}_\phi(q,\nu) \equiv \ln \left[ \int \nu(\mathrm{d}x,\mathrm{dy}) \, \phi'(x)^2 \left[\phi\left(q x+ \sqrt{1-q^2} y\right) - \phi(x)\right]^2\right]$,
\item $t_\phi(q,\nu) \equiv \int \nu(\mathrm{d}x,\mathrm{dy}) \, x \, \phi'(x) \left[\phi(x) - \phi\left(q x+ \sqrt{1-q^2} y\right) \right]$,
\item $f_{q}$ is a function from $\bbR^2$ to $\bbR$ defined by:
\begin{align}
f_q(x,y) &\equiv \phi'\left(x\right) ^2 -\phi''\left(x\right) \left[\phi\left(q x+ \sqrt{1-q^2} y\right)- \phi\left(x\right)\right],
\end{align}
\item Let $\bz \in \bbR^{n \times m}$ an i.i.d.\ standard Gaussian matrix, and $\bY \in \bbR^{m \times 2}$ with components taken i.i.d.\ from a probability measure $\nu$ on $\bbR^2$. 
Let $D^{(\nu,q)}$ the diagonal matrix of size $m$ with elements $D^{(\nu,q)}_{\mu} = f_q(Y_\mu)$.
We define $\mu_{\alpha,\phi}[q,\nu]$ as the asymptotic spectral measure of $\bz D^{(\nu,q)} \bz^\intercal / m$.
\item $\kappa_{\alpha,\phi}(q,\nu) \equiv \int \mu_{\alpha,\phi}[q,\nu](\mathrm{d}x) \ln \left|x - t_\phi(q,\nu)\right| $,
\end{itemize}
We moreover assume that the counterpart of Assumption~\ref{hyp:technical} for $L_2$ holds.
Then one has:
\begin{align}
\lim_{n \to \infty} \frac{1}{n}\ln \EE \ \mathrm{Crit}_{n,L_2}(B,Q) &= \frac{1+\ln \alpha}{2} \\
&\hspace{-1cm} + \sup_{q \in Q} \sup_{\nu \in \mathcal{M}_\phi(B,q)} \left[\frac{1}{2} \ln(1-q^2) -\frac{1}{2} \mathcal{E}_\phi(q,\nu) + \kappa_{\alpha,\phi}(q,\nu) - \alpha H(\nu|\mu_G) \right]. \nonumber
\end{align}
\end{theorem}
The proof of Theorem~\ref{thm:annealed_nosignal} is presented in Section~\ref{sec:annealed}. The proof  
of Theorem~\ref{thm:annealed_signal} is a straightforward generalization, and sketched in Appendix~\ref{subsec_app:L2}.
The variational problems in Theorems~\ref{thm:annealed_nosignal} and \ref{thm:annealed_signal} are challenging, as they imply an optimization on a set of measures, and they involve transforms of this measure that are very hard to access numerically. In Section~\ref{sec:variational} we present 
a drastic simplification: a heuristic calculation that allows one to reduce the supremum over the probability measure $\nu$ to a much more straightforward optimization over a small number of parameters. 

As we have already stressed, the annealed complexity, although interesting in itself, is generically not representative of the landscape corresponding to a given typical instance of the empirical risk. In order to 
obtain the value of the quenched complexity we use the replicated Kac-Rice method, which is an extension to non-Gaussian functions of the one developed in \cite{ros2019complex}. Although the replica method is non-rigorous, it is considered an exact method in theoretical physics and it has been proven to give correct results for spin-glasses and inference problems \citep{Talagrand2006,barbier2017phase}.
We have obtained an explicit formula\footnote{Here we used a replica symmetric structure, 
which is correct in many cases, and a very good approximation in others were replica symmetry has to be broken \citep{mezard1987spin}.} for the quenched complexity of $L_1$ and $L_2$ at fixed values of the empirical risk, and overlap with the solution (in the $L_2$ case).

For $L_1$, using the notations of Theorem~\ref{thm:annealed_nosignal} we have:
\begin{result}[Quenched complexity of $L_1$]\label{result:quenched_nosignal}
Let $B \subseteq \bbR$ an open interval.
{\small
\begin{align*}
    &\lim_{n \to \infty}  \frac{1}{n} \, \EE \ln \mathrm{Crit}_{n,L_1}(B) =  \frac{\ln \alpha - \alpha \ln 2\pi}{2} + \sup_{\substack{\nu \in {\cal M}_\phi(B) \\ q \in (0,1)}} \extr
\Bigg\{\kappa_{\alpha,\phi}(\nu,C) + \frac{1-\alpha}{2} \ln(1-q)  \\
 & + \frac{1-\alpha q}{2 (1-q)} - \int \nu(\mathrm{d}\lambda) g(\lambda) - \frac{A \hat{A}-a \hat{a}}{2} + C (q \hat{c} - \hat{C}) - \frac{1}{2} \ln[A-a] - \frac{a}{2 (A-a)} + \alpha \int_{\bbR^4} {\cal D}\bxi \ln I(\bxi) \Bigg\}.
\end{align*}
}
Here, $\bxi = (\xi_q,\xi_a,\xi_c,\xi_c')$ and ${\cal D}\bxi$ is the standard Gaussian probability measure on $\bbR^4$. The $\extr$ denotes extremization with respect to all variables $(A,\hat{A},a,\hat{a},C,\hat{C},\hat{c},\{g(\lambda)\})$. We denoted
\begin{align*}
    I(\bxi) &\equiv \int_\bbR \mathrm{d} \lambda \, e^{-\frac{\lambda^2}{2 (1-q)} +  \frac{g(\lambda)}{\alpha} + \frac{\hat{A} - \hat{a}}{2\alpha} \phi'(\lambda)^2 + \frac{\hat{C} - \hat{c}}{\alpha} \phi'(\lambda) \lambda + \frac{\sqrt{q}}{1-q} \xi_q \lambda + \sqrt{\frac{\hat{a}}{\alpha}} \xi_a \phi'(\lambda) + \sqrt{\frac{\hat{c}}{2\alpha}} [\phi'(\lambda)(\xi_c + i \xi_c') + \lambda (\xi_c - i\xi_c') ] }.
\end{align*}
\end{result}
In this formula, the notation $\mathrm{extr}$ denotes that one should set the partial derivatives with respect to the involved variables to zero.
This notation arises from the replica calculation, which mixes saddle-point computations with Lagrange multipliers 
associated to certain constraints, and the precise meaning of this extremization (as a supremum or infimum) would have to be clarified by a more rigorous method. 
On a numerical point of view, one would have to solve the associated saddle-point equations, so that 
this precise meaning is not crucial for applications.
We can state a very similar result for $L_2$:
\begin{result}[Quenched complexity of $L_2$]\label{result:quenched_signal}
Let $B \subseteq \bbR, Q \subseteq (-1,1)$ two open intervals and:
\begin{itemize}
    \item For $m \in (-1,1)$, $\mathcal{M}_\phi(B,m)$ is the space of probability measures $\nu$ on $\bbR^2$ that satisfy:
    \begin{align}
    \begin{cases}
     &\frac{1}{2} \int \nu(\mathrm{d}\lambda^0, \mathrm{d}\lambda) [\phi(\lambda) - \phi(\lambda^0)]^2 \in B, \\
    &\int \nu(\mathrm{d}\lambda^0, \mathrm{d}\lambda) \phi'(\lambda) [\phi(\lambda)-\phi(\lambda^0)] (\lambda^0 - m \lambda) = 0.
    \end{cases}
    \end{align}
    \item Let $f(x,y) \equiv \phi''(y)[\phi(y) - \phi(x)] + \phi'(y)^2$. Let $\bz \in \bbR^{n \times m}$ an i.i.d.\ standard Gaussian matrix, and $\bY \in \bbR^{m \times 2}$ with components taken i.i.d.\ from a probability measure $\nu$ on $\bbR^2$. 
    Let $D^{(\nu)}$ the diagonal matrix of size $m$ with elements $D^{(\nu)}_{\mu} = f(Y_\mu)$.
    We define $\mu_{\alpha,\phi}[\nu]$ as the asymptotic spectral measure of $\bz D^{(\nu)} \bz^\intercal / m$.
\item $\chi_{\alpha,\phi}(\nu,C) \equiv \int \mu_{\alpha,\phi}[\nu](\mathrm{d}x) \ln |x-C|$.
\end{itemize}
One has:
\begin{align*}
      &\lim_{n \to \infty}\frac{1}{n} \EE \ln \mathrm{Crit}_{n,L_2}(B,Q) = \sup_{\substack{m \in Q \\ q \in (0,1)}} \sup_{\nu \in {\cal M}_\phi(B,m)} \extr \left[\frac{\ln \alpha - \alpha \ln 2\pi}{2} + \chi_{\alpha,\phi}(\nu,C)  \right. \\
  &\left. \hspace{0.5cm}  + \frac{1- \alpha q - m^2}{2(1-q)} + \frac{1-\alpha}{2} \ln(1-q)- \frac{1}{2} \ln (A-a)  - \frac{a}{2 (A-a)} - \frac{A \hat{A}}{2} + \frac{a \hat{a}}{2}  \right. \\
  & \left.  \hspace{0.5cm} - C_0 \hat{C}_0 - C \hat{C} + c \hat{c}  - \int \nu(\mathrm{d}\lambda^0,\mathrm{d}\lambda) g(\lambda^0,\lambda) + \alpha \int_{\bbR^4 \times \bbR} \mathcal{D}\bxi \, \mathcal{D}\lambda^0 \ln I(\lambda^0,\bxi) \right].
\end{align*}
The extremum is made over all the variables $(A,a,C_0,C,c,\hat{A},\hat{a},\hat{C},\hat{c},\hat{C}_0,\{g(\lambda^0,\lambda)\})$.
${\cal D}$ denotes the standard Gaussian measure, and the variables $C_0,c,C$ are related by the additional constraint
\begin{align*}
    - m (1-q) C_0 - (q-m^2) C + (1-m^2) c &= 0.
\end{align*}
$I(\lambda^0,\bxi)$ is defined as, with $\bxi \equiv (\xi_q,\xi_a,\xi_c,\xi_c')$:
\begin{align*}
   I(\lambda^0,\bxi) &\equiv \int_\bbR \mathrm{d}\lambda \, e^{\frac{m}{1-q} \lambda^0 \lambda - \frac{\lambda^2}{2(1-q)} + \frac{\sqrt{q-m^2}}{1-q}\xi_q\lambda + \frac{g(\lambda^0,\lambda)}{\alpha} + \frac{\hat{C}_0}{\alpha} \phi'(\lambda)[\phi(\lambda)-\phi(\lambda^0)]\lambda^0 + \frac{\hat{C}-\hat{c}}{\alpha} \lambda \phi'(\lambda) [\phi(\lambda)-\phi(\lambda^0)] } \\
  & \hspace{0.5cm} e^{\frac{\hat{A}-\hat{a}}{2\alpha} \phi'(\lambda)^2 [\phi(\lambda) - \phi(\lambda^0)]^2 + \sqrt{\frac{\hat{a}}{\alpha}} \xi_a \phi'(\lambda)[\phi(\lambda)-\phi(\lambda^0)] + \sqrt{\frac{\hat{c}}{2\alpha}}\left[\phi'(\lambda)[\phi(\lambda)-\phi(\lambda^0)](\xi_c + i \xi_c') + \lambda(\xi_c-i\xi_c')\right]}.
\end{align*}
\end{result}
The derivation of Result~\ref{result:quenched_nosignal} is given in Section~\ref{sec:quenched}. Result~\ref{result:quenched_signal} can be derived by a straightforward generalization of this computation, see Appendix~\ref{subsec_app:L2}.

\subsection{Conclusion, outcomes and perspectives}\label{subsec:conclusion_outcome}

We have obtained analytical results for the annealed and quenched complexities of statistical models with non-Gaussian loss functions
arising in generalized linear estimation and simple models 
of glasses and neural networks. Our method is versatile and can be easily extended to other cases. We describe in Appendix~\ref{subsec_app:other_models} three other inference models to which it applies: binary linear classification, a mixture of two Gaussians, and a simple model of unsupervised learning.

As a sanity check of our results, we have analytically verified
by explicit solution that for a linear activation function, 
the annealed complexities of $L_1$ is null. It is again a tedious but straightforward computation to check that the annealed complexity of $L_1$ with 
a quadratic activation $\phi(x) = x^2$ is also null, as the number of critical points in this case is linear with $n$. Note that for $L_2$, even the case of a linear activation is non trivial,
as shown in the very recent analysis of \cite{1911.12452}.

Our results allow for a complete characterization of the empirical loss landscapes of generalized linear models. The main issue ahead is determining for which class of functions $\phi$ and in which regimes (e.g. values of $\alpha$), the annealed and quenched complexities become positive, i.e.\ when the associated landscape is rough. This will allow to study the connection between landscape properties and dynamics induced by local algorithms. In particular, it will shed light on the relationship 
between the roughness of the empirical loss landscape and the existence of ``hard'' phases in the learning of generalized linear models \citep{barbier2017phase}. It will also provide an interesting benchmark for obtaining the algorithmic thresholds of gradient descent (and variants) only through the knowledge of the landscape properties \citep{mannelli2019afraid,mannelli2018marvels}. 
Based on ongoing works, we can for instance conjecture the existence of a rough landscape for small enough $\alpha$ in
phase retrieval \citep{lucibello2019generalized} and retarded learning \citep{engel2001statistical}.

Addressing these questions requires additional work, which is beyond the scope 
of this paper. We have explained in Section~\ref{sec:variational} how to make tractable the variational problem associated to the  
annealed and quenched complexity. Its analysis for specific models is an ongoing direction of research and will be presented elsewhere. 

Another important extension of our results consists in counting the critical points of a fixed index (i.e.\ with a fixed number of negative directions in the spectrum of the Hessian).
This would provide additional interesting information, in particular it would allow to differentiate local minima from the other critical points of the landscape, as was done for spin glass models in \cite{1003.1129}. Such a counting would require to understand the large deviations properties of the 
eigenvalues of the Hessian arising for generalized linear models, a random matrix problem that is hard but hopefully tractable by building on recent developments \citep{maida2007large,1711.05424}.
This is an ongoing work that we are pursuing, with already promising results.

As final note, it is an open problem to generalize our methods to neural network models with many nodes and hidden layers; the random matrix analysis of the Hessian in this case is a particularly exciting challenge.
\section{Proof of Theorem~\ref{thm:annealed_nosignal} for the annealed complexity}\label{sec:annealed}

In this section we prove Theorem~\ref{thm:annealed_nosignal}. 
The technique leverages the Kac-Rice formula and Sanov's theorem on the large deviations of the empirical measure of i.i.d.\ variables. First we precise our hypotheses on $\phi$, that we will take in the following set of ``well-behaved'' activation functions:
\begin{defn}\label{def:well_behaved} $\phi : \bbR \to \bbR$ is ``well-behaved'' if it is of class ${\cal C}^3$ and if,
for $y \sim {\cal N}(0,1)$, the random variable $a = \phi'(y)$ admits a continuous probability density in a neighborhood of $a=0$.
\end{defn}

\subsection{The Kac-Rice formula}

The first step is to apply the Kac-Rice formula to the random function $L_1$:
\begin{lemma}[Kac-Rice formula]\label{lemma:Kac_Rice}
For any $\bx \in \mathbb{S}^{n-1}$, denote $\mathrm{grad} L_1(\bx)$ and $\mathrm{Hess} L_1(\bx)$ the (Riemannian) gradient and Hessian of $L_1$ at the point $\bx$.
Then $\mathrm{grad} L_1(\bx)$ has a well defined density (on the tangent space $T_\bx \mathbb{S}^{n-1} \simeq \bbR^{n-1}$) in a neighborhood of zero, that we denote $\varphi_{\mathrm{grad} \, L_1(\bx)}$. Denote $\mu_{\mathbb{S}}$ the usual surface measure on $\mathbb{S}^{n-1}$. One has:
\begin{align*}
\EE \, \mathrm{Crit}_{n,L_1}(B) &= \int_{\mathbb{S}^{n-1}} \varphi_{\mathrm{grad} \, L_1(\bx)}(0) \EE \left[\mathds{1}_{L_1(\bx) \in B}\left|\det \mathrm{Hess} L_1(\bx)\right| \Big| \mathrm{grad}\, L_1(\bx) = 0\right] \mu_{\mathbb{S}}\left(\mathrm{d} \bx\right).
\end{align*}
\end{lemma}
The proof of this lemma uses necessary conditions for a random function to be a.s.\ Morse\footnote{A \emph{Morse} function is a function whose critical points are all non-degenerate.}, that are stated in \cite{azais2009level}. The details are given in Appendix~\ref{subsec_app:proof_lemma_KR}. 

\subsection{The complexity at finite \texorpdfstring{$n$}{n}}\label{subsec:complexity_finiten}

In this section we state the result of the Kac-Rice method. For $\by \in \bbR^m$, let $\Lambda(\by) \in \bbR^{m \times m}$:
\begin{align}\label{eq:def_Lambda}
\Lambda(\by) \equiv \left(\mathrm{I}_m - \frac{\phi'(\by)\phi'(\by)^\intercal}{\norm{\phi'(\by)}^2}\right)D(\by) \left(\mathrm{I}_m - \frac{\phi'(\by)\phi'(\by)^\intercal}{\norm{\phi'(\by)}^2}\right),
\end{align}
in which we (abusively) denote $\phi'(\by) \equiv (\phi'(y_\mu))_{\mu=1}^m$, and $D(\by) \in \bbR^{m \times m}$ the diagonal matrix with elements $D(\by)_\mu = D(y_\mu) = \frac{n}{m} \phi''(y_\mu)$. The main result of this section is:
\begin{lemma}[Complexity at finite $n$]\label{lemma:annealed_finite_n_L1}
\begin{align*}
\EE \ \mathrm{Crit}_{n,L_1}(B) &= \mathcal{C}_n e^{n\frac{1+\ln \alpha}{2}} \EE_{\by} \left[\mathds{1}_{\frac{1}{m} \sum_{\mu} \phi(y_\mu) \in B} \, \, e^{- \frac{n-1}{2}\ln \left(\frac{1}{m} \sum_{\mu}\phi'(y_\mu)^2\right)  } \EE_{\bz} \left[|\det H_n^\Lambda(\by) | \right] \right],
\end{align*}
in which $\mathcal{C}_n$ is exponentially trivial, meaning $\lim_{n \to \infty} (1/n)\ln \mathcal{C}_n = 0$. The variable $\by \in \bbR^m$ follows $\mathcal{N}(0,\mathrm{I}_m)$, and $\bz \in \bbR^{(n-1)\times m}$ has i.i.d.\ standard Gaussian matrix elements, independent of $\by$. $H_n^\Lambda(\by)$ is a square matrix of size $(n-1)$ with the following distribution :
\begin{align}\label{eq:def_H}
H_n^\Lambda(\by) &\overset{d}{=} \frac{1}{n} \bz \Lambda(\by)\bz^\intercal - \left[\frac{1}{m} \sum_{\mu=1}^m y_\mu \phi' \left(y_\mu\right) \right]\mathrm{I}_{n-1}.
\end{align}
\end{lemma}
The rest of the section is devoted to the proof of Lemma~\ref{lemma:annealed_finite_n_L1}. We start from the result of Lemma~\ref{lemma:Kac_Rice}.
The following proposition specifies the joint distribution of $\left(L_1(\bx),\text{grad} L_1(\bx),\text{Hess} L_1(\bx)\right)$:
\begin{proposition}[Distribution of the gradient and Hessian]\label{prop:joint_distribution_L1}
Let $\bx \in \mathbb{S}^{n-1}$. Then \\ $\left(L_1(\bx),\mathrm{grad} \, L_1(\bx),\mathrm{Hess} \, L_1(\bx)\right)$ follows the  following joint distribution:
\begin{subnumcases}{\label{eq:distribution_joint_L1}}
\label{eq:dist_L1}
L_1(\bx) \overset{d}{=} \frac{1}{m}{\sum_{\mu=1}^m} \phi(y_\mu), & \\
\label{eq:dist_grad_L1}
\mathrm{grad} \, L_1(\bx) \overset{d}{=} \frac{1}{m} {\sum_{\mu=1}^m} \phi'(y_\mu) \bz_\mu , & \\
\label{eq:dist_Hess_L1}
\mathrm{Hess} \, L_1(\bx) \overset{d}{=} \frac{1}{m} {\sum_{\mu=1}^m} \phi''(y_\mu) \bz_\mu \bz_\mu^\intercal - \left[\frac{1}{m} \sum_{\mu=1}^m y_\mu \phi'(y_\mu)\right] \mathrm{I}_{n-1}, &
\end{subnumcases}
in which $\by = (y_\mu)_{\mu=1}^m \sim\mathcal{N}(0,\mathrm{I}_m)$, $(\bz_\mu)_{\mu=1}^m \overset{i.i.d.\ }{\sim}\mathcal{N}(0,\mathrm{I}_{n-1})$, and all $\{y_\mu,\bz_\nu\}$ are independent.
We identified in these equations the tangent spaces $T_\bx \, \mathbb{S}^{n-1}$ with $\bbR^{n-1}$. 
\end{proposition}
\begin{proof}
Denote $P_\bx^\perp$ the orthogonal projection on $\{\bx\}^\perp$. For a smooth function $f : \mathbb{S}^{n-1} \to \bbR$, $\nabla f$ and $\nabla^2 f$ are its Euclidean gradient and Hessian. The Riemannian structure on $\mathbb{S}^{n-1}$ induces the gradient and Hessian of $f$ as $\mathrm{grad}\, f(\bx) = P_\bx^\perp \nabla f$ and $\mathrm{Hess}\, f(\bx) = P_\bx^\perp \nabla^2 f P_\bx^\perp - (\bx \cdot \nabla f(\bx)) P_\bx^\perp$. Applying these formulas yields: 
\begin{align}\label{eq:gradient_L1_derivation}
\mathrm{grad} \, L_1(\bx) &= \frac{1}{m} \sum_{\mu=1}^m (P_\bx^\perp \bxi_\mu) \phi'(\bxi_\mu \cdot \bx), \\
\label{eq:hessian_L1_derivation}
\mathrm{Hess} \, L_1(\bx) &= \frac{1}{m} \sum_{\mu=1}^m \phi''(\bxi_\mu \cdot \bx) \left(P_\bx^\perp \bxi_\mu\right) \left(P_\bx^\perp \bxi_\mu\right)^\intercal - \left[\frac{1}{m} \sum_{\mu=1}^m (\bxi_\mu \cdot \bx) \phi'(\bxi_\mu \cdot \bx)\right] P_\bx^\perp.
\end{align}
Letting $y_\mu \equiv \bxi_\mu \cdot \bx$ and $\bz_\mu \equiv P_\bx^\perp \bxi_\mu$ (identified to an element of $\bbR^{n-1}$) yields the result.
\end{proof}
The joint distribution of eq.~\eqref{eq:distribution_joint_L1} is invariant with respect to $\bx$, thus we can chose $\bx$ to be the North pole $\bx = \bbe_n = (\delta_{i,n})_{i=1}^n$. With $\omega_n \equiv 2 \pi^{n/2} / \Gamma(n/2)$ the volume of $\mathbb{S}^{n-1}$, we obtain from Lemma~\ref{lemma:Kac_Rice}:
\begin{align}\label{eq:annealed_L1_invariant}
\EE \, \text{Crit}_{n,L_1}(B) &= \omega_n  \varphi_{\text{grad} L_1(\bbe_n)}(0) \EE \left[\left|\det \text{Hess} L_1(\bbe_n)\right| \mathds{1}_{L_1(\bbe_n) \in B} \Big| \text{grad}\, L_1(\bbe_n) = 0\right].
\end{align}
Removing the $\bbe_n$ indication and conditioning on the distribution of $\by$ in eq.~\eqref{eq:distribution_joint_L1}, we reach:
\begin{align*}
\EE \, \text{Crit}_{n,L_1}(B) &= \omega_n \EE_{\by} \left[\mathds{1}_{\frac{1}{m} \sum_{\mu} \phi(y_\mu) \in B} \, \varphi_{\text{grad} L_1 | \by}(0) \,  \EE_{\bz} \left[\left|\det \text{Hess} L_1\right|   \Big| \text{grad}\, L_1 = 0, \by\right] \right].
\end{align*}
Once conditioned on $\by$, eq.~\eqref{eq:dist_grad_L1} describes a Gaussian density so we can directly compute:
\begin{align}
\omega_n \varphi_{\text{grad} L_1 | \by}(0) &= \frac{2 \pi^{n/2}}{\Gamma(n/2)} \exp\left[-\frac{n-1}{2} \ln \left(\frac{2\pi}{m^2} \sum_{\mu=1}^m \phi'(y_\mu)^2\right)\right], \nonumber \\
&= {\cal C}_n \exp\left[\frac{n}{2} + \frac{n}{2} \ln \frac{m}{n} - \frac{n-1}{2}\ln \left(\frac{1}{m} \sum_{\mu=1}^m\phi'(y_\mu)^2\right)  \right],
\end{align}
in which $\ln \mathcal{C}_n = \smallO_n(n)$ (using Stirling's formula).
The conditioning of the Hessian by $\mathrm{grad}\, L_1 = 0$ at fixed $\by$ reduces to a linear conditioning on $\bz$. One thus obtains by classical Gaussian conditioning:
\begin{align}
 \EE_{\bz} \left[\left|\det \mathrm{Hess} L_1\right| \Big| \mathrm{grad}\, L_1 = 0, \by\right] &= \EE_{\bz} \left[|\det H_n^\Lambda(\by) | \right],
\end{align}
in which $H_n^\Lambda(\by)$ is defined by eq.~\eqref{eq:def_H}. This ends the proof of Lemma~\ref{lemma:annealed_finite_n_L1}.

\subsection{Large deviations}\label{subsec:large_deviations_annealed_L1}

This section is devoted to the end of the proof of Theorem~\ref{thm:annealed_nosignal}.
Denote $\nu_\by^m \equiv \frac{1}{m} \sum_{\mu=1}^m \delta_{y_\mu}$ the empirical distribution of $\by$. We take the notations of Theorem~\ref{thm:annealed_nosignal} and Lemma~\ref{lemma:annealed_finite_n_L1}. 
We first state an important lemma on the concentration of  $\EE_{\bz} \left[|\det H_n^\Lambda(\by) | \right]$\footnote{In the proofs of this section we assume that $x\phi'(x)$ and $\phi''(x)$ are bounded. As one can always smoothly truncate the largest values of $\phi$ without affecting the complexity, this does not remove any generality to our results.}:
\begin{lemma}\label{lemma:log_potential_L1}
Assume that Assumption~\ref{hyp:technical} holds.
There exists $\eta > 0$ such that for all $t > 0$:
\begin{align}
  \lim_{n\to \infty}\frac{1}{n^{1+\eta}} \ln \mathbb{P}\left[\left|\frac{1}{n} \ln \EE_{\bz} \left[ \left|\det H_n^\Lambda(\by) \right| \right] - \kappa_{\alpha,\phi}\left(\nu_\by^m,t_\phi(\nu_\by^m)\right) \right| \geq t\right]  &= -\infty.
\end{align}
\end{lemma}
The proof of Lemma~\ref{lemma:log_potential_L1} is detailed in Appendix~\ref{subsec:proof_lemma_log_pot_L1}, 
and is the part of our proof that relies on the technical Assumption~\ref{hyp:technical}.
Note that we expect this result to actually be valid up to $\eta = 1$, as the large deviations of the spectral distribution of random matrices is typically on the $n^2$ scale \citep{arous1997large,hiai1998eigenvalue}. 
The following moment condition, proven in Appendix~\ref{subsec:proof_lemma_moment_condition}, will be important:
\begin{lemma}\label{lemma:varadhan_well_posed}For every $\gamma \in (1,\alpha)$ we have:
\begin{subnumcases}{}
\limsup_{n \to \infty} \frac{1}{n} \ln \EE_\by \left[e^{\gamma n\left[-\frac{1}{2}\ln \left(\frac{1}{m} \sum_{\mu}\phi'(y_\mu)^2\right) +  \kappa_{\alpha,\phi}\left(\nu_\by^m,t_\phi(\nu_\by^m)\right)\right]} \right] < + \infty, & \\
\limsup_{n \to \infty} \frac{1}{n} \ln \EE_\by \left[e^{\gamma n\left[-\frac{1}{2}\ln \left(\frac{1}{m} \sum_{\mu}\phi'(y_\mu)^2\right) +  \frac{1}{n} \ln \EE_{\bz} \left[ \left|\det H_n^\Lambda(\by) \right| \right]\right]} \right] < +\infty. & 
\end{subnumcases}
\end{lemma}
Then we can conclude using Sanov's large deviation principle \citep{sanov1958probability,dembo2010large}, and Varadhan's lemma.
Using our previous lemmas, we obtain the statement of Theorem~\ref{thm:annealed_nosignal}:
\begin{align}\label{eq:final}
\lim_{n \to \infty} \frac{1}{n}\ln \EE \ \mathrm{Crit}_{n,L_1}(B) &= \sup_{\nu \in {\cal M}_\phi(B)} \,  \left[\frac{1 + \ln \alpha}{2}  - \frac{\mathcal{E}_\phi(\nu)}{2} + \kappa_{\alpha,\phi}(\nu,t_\phi(\nu)) - \alpha H(\nu|\mu_G)\right].
\end{align}
The proof of eq.~\eqref{eq:final} is detailed in Appendix~\ref{subsec:proof_eq_before_Varadhan}.
\section{Towards a numerical solution to the variational problem}\label{sec:variational}

\subsection{The logarithmic potential of \texorpdfstring{$\mu_{\alpha,\phi}[\nu]$}{mu[nu]}}

Let $\nu \in {\cal M}(\bbR)$. 
The Stieltjes transform $g(z) \equiv \int \mu(\mathrm{d} t) (t-z)^{-1}$ of $\mu_{\alpha,\phi}[\nu]$ is given by the unique solution in $\bbC_+$ to the implicit equation (as shown for instance in \cite{silverstein1995empirical}):
\begin{align}\label{eq:stieltjes_mualphanu}
\forall z \in \bbC_+, \quad g(z) &= -\left[z-  \alpha \int \frac{\phi''(t)}{\alpha+\phi''(t) g(z)} \nu(\mathrm{d}t) \right]^{-1}.
\end{align}
For any $\mu \in {\cal M}(\bbR)$ and $t \in \bbR$ we define the \emph{logarithmic potential} as $U[\mu](t) \equiv \int \mu(\mathrm{d}x) \ln |x-t|$. It is well defined with values in $\bbR \cup \{\pm\infty\}$, see \cite{faraut2014logarithmic} for a review on this subject.
Our goal is to numerically compute $U[\mu](t)$ for $\mu = \mu_{\alpha,\phi}[\nu]$ and an arbitrary $t \in \bbR$, see Theorem~\ref{thm:annealed_nosignal}. For clarity, we will write $\mu$ for $\mu_{\alpha,\phi}[\nu]$ for the remainder of this section.
 Let us define for any $z \in \bbC_+$, $G(z) \equiv \int \mu(\mathrm{d}x) \ln(z-x)$. $G(z)$ is well defined and holomorphic on $\bbC_+$.
 Moreover, from the Chapter II of \cite{faraut2014logarithmic}, we know that
  $U[\mu](t) = \lim_{\epsilon \to 0^+}\mathrm{Re}\, G(t+i\epsilon)$.

 From this it is clear that in order to get the logarithmic potential in the bulk, 
 we need to be able to evaluate $G(z)$ for $z \in \bbC_+$. Define, for $z,g \in \bbC_+$\footnote{$g \in \bbC_+$, and (since $\alpha > 1$) $\alpha + \phi''(\lambda) g \in \bbC \backslash (-\infty,1]$, thus we can use the principal determination of the logarithm.}:
 \begin{align}\label{eq:def_F}
    F(z,g) &\equiv - \ln (-g) - z g + \alpha \int \nu(\mathrm{d}\lambda) \ln (\alpha + \phi''(\lambda)g) - 1 - \alpha \ln \alpha.
 \end{align}
 At any fixed $z$, $F(z,g)$ is an holomorphic function of $g$ on $\bbC_+$. Its Wirtinger derivative is:
 \begin{align}
  \frac{\partial F}{\partial g}(z,g) &= - \frac{1}{g} - z + \alpha \int \nu(\mathrm{d}\lambda) \frac{\phi''(\lambda)}{\alpha + \phi''(\lambda) g}.
 \end{align}
Thus $g(z)$ (the Stieltjes transform of $\mu$, cf.\ eq.~\eqref{eq:stieltjes_mualphanu}) is the \emph{only} $g \in \bbC_+$ such that $\frac{\partial F}{\partial g}(z,g) = 0$.
 Moreover, by definition $g(z)$ is an holomorphic function on $\bbC_+$ with values 
 in $\bbC_+$.
 We can thus apply the usual composition of derivatives and obtain:
 \begin{align}
    \frac{\mathrm{d}F}{\mathrm{d}z}(z,g(z)) &= - g(z). 
 \end{align}
 Furthermore we know $\frac{\mathrm{d} G}{\mathrm{d}z} = -g(z)$.
 Computing the remaining constant by investigating the limit $\mathrm{Re}[z] \to \infty$, we reach 
 that $G(z) = F(z,g(z))$ for every $z \in \bbC_+$. We thus have the crucial relation:
 \begin{align}\label{eq:U_bulk}
    \forall t \in \bbR, \quad U[\mu](t) &= \lim_{\epsilon \to 0^+} \mathrm{Re} \, F(t+i\epsilon,g(t+i\epsilon)).
  \end{align}
  This allows for an efficient numerical derivation of the logarithmic potential of $\mu_{\alpha,\phi}[\nu]$, as we will see in more details below. 
  
\subsection{Heuristic derivation of the simplified fixed point equations corresponding to  \texorpdfstring{Theorem~\ref{thm:annealed_nosignal}}{Theorem 1}}

We present here an heuristic derivation of scalar fixed point equations for the numerical resolution of Theorem~\ref{thm:annealed_nosignal}.
This technique could be easily extended to Theorem~\ref{thm:annealed_signal} as well as the quenched calculations presented afterwards, but we restrict to this simpler case for the sake of the presentation. 

\subsubsection{Expressing \texorpdfstring{$\kappa_{\alpha,\phi}(\nu,t)$}{K[v]}}

From eq.~\eqref{eq:U_bulk} we know that for every $t \in \bbR$:
 \begin{align*}
   \kappa_{\alpha,\phi}(\nu,t) = \lim_{\epsilon \to 0^+} \mathrm{Re}\, &\Big[- \ln (-g(t+i\epsilon)) - (t + i \epsilon) g(t+i\epsilon) + \alpha \int \nu(\mathrm{d}\lambda) \ln \left[\alpha + \phi''(\lambda)g(t+i\epsilon)\right]  \\
   &- 1 - \alpha \ln \alpha\Big]. 
 \end{align*}
For every $t \in \bbR$, $g(t+i\epsilon)$ is the only solution in $\bbC_+$ to the partial derivative of the previous equation:
 \begin{align}\label{eq:condition_g}
  - \frac{1}{g} - (t+i \epsilon) + \alpha \int \nu(\mathrm{d}\lambda) \frac{\phi''(\lambda)}{\alpha + \phi''(\lambda) g} &= 0.
 \end{align}
 So heuristically, we can write that for a small enough $\epsilon$:
 \begin{align}\label{eq:kappa_heuristic}
   \kappa_{\alpha,\phi}(\nu,t) = \extr_{g \in \bbC_+} &\Big[- \ln |g| - t g_r + \epsilon g_i + \alpha \int \nu(\mathrm{d}\lambda) \ln \left|\alpha + \phi''(\lambda)g\right| - 1 - \alpha \ln \alpha\Big],
 \end{align}
 with $g = g_r + i g_i$ (in practice one iterates over $g_r$ and $g_i$ successively).
 
\subsubsection{Heuristic solution to Theorem~\ref{thm:annealed_nosignal}}
We start from the result of Theorem~\ref{thm:annealed_nosignal}.
For a function $f$, we write $\EE[f(X)] \equiv \int \nu(\mathrm{d}t) f(t)$. We introduce Lagrange multipliers to fix the conditions $\EE[\phi(X)] \in B$, and we fix the values of $\EE[\phi'(X)^2]$ and $\EE[X \phi'(X)]$.
We obtain: 
\begin{align}
\lim_{n \to \infty} \frac{1}{n}&\ln \EE \ \mathrm{Crit}_{n,L_1}(B) = \sup_{\substack{l \in B \\ \nu \in {\cal M}(\bbR)}} \extr_{\lambda_0,\lambda_1,\lambda_2} \sup_{A,t} \Big[\frac{1 + \ln \alpha}{2} - \frac{1}{2} \ln A + \lambda_0 l + \lambda_1 A + \lambda_2 t  \nonumber \\
 & +  \kappa_{\alpha,\nu}(\nu,t) - \alpha H(\nu|\mu_G) - \lambda_0 \EE[\phi(X)] - \lambda_1\EE[\phi'(X)^2] - \lambda_2 \EE[X \phi'(X)] \Big].
\end{align}
Note that now the supremum over $\nu$ is unconstrained over the set ${\cal M}(\bbR)$ of probability distributions.
We now make use of eq.~\eqref{eq:kappa_heuristic} to write, with  $2 K(\alpha) \equiv -1 + \ln \alpha - 2 \alpha \ln \alpha$ and a small $\epsilon > 0$:
\begin{align}
\lim_{n \to \infty} \frac{1}{n}&\ln \EE \ \mathrm{Crit}_{n,L_1}(B) = \sup_{\substack{l \in B \\ \nu \in {\cal M}(\bbR)}} \extr_{\substack{\{\lambda_i\},A,t \\ g \in \bbC_+}} \Big[K(\alpha) - \frac{1}{2} \ln A + \lambda_0 l + \lambda_1 A + \lambda_2 t \nonumber \\
 & - \ln |g| - t \mathrm{Re}[g] + \epsilon \mathrm{Im}[g] + \alpha \int \nu(\mathrm{d}\lambda) \ln \left|\alpha + \phi''(\lambda)g\right| - \alpha H(\nu|\mu_G) - \lambda_0 \EE[\phi(X)] \nonumber \\
  &- \lambda_1\EE[\phi'(X)^2] - \lambda_2 \EE[X \phi'(X)] \Big].
\end{align}
For any scalar function $F$, the maximum $\sup_\nu \left[\EE [F(X)] - \alpha H(\nu|\mu_G) \right]$ is attained in $\nu^*$ with density proportional to $e^{-x^2/2+F(x)/\alpha}$, which is called the \emph{Gibbs measure} in statistical physics. This gives  (${\cal D}$ is the standard Gaussian measure on $\bbR$):
\begin{align}
\sup_{\nu \in {\cal M}(\bbR)} \left[\EE [F(X)] - \alpha H(\nu|\mu_G) \right] &= \alpha \ln \left[\int_\bbR {\cal D}x \, e^{F(x)/\alpha}\right].
\end{align}
Plugging this into our previous equation for the annealed complexity yields:
\begin{align}
\lim_{n \to \infty} \frac{1}{n}&\ln \EE \ \mathrm{Crit}_{n,L_1}(B) = \sup_{l \in B} \extr_{\substack{\{\lambda_i\} \\ g \in \bbC_+}} \sup_{A,t} \Bigg\{K(\alpha) + \lambda_0 l + \lambda_1 A + \lambda_2 t \nonumber \\
 & \hspace{1cm} - \frac{1}{2} \ln A  - \ln |g| - t \mathrm{Re}[g] + \epsilon \mathrm{Im}[g]  \\ 
 & + \alpha \ln \left[\int_\bbR {\cal D}x \exp\left\{- \frac{\lambda_0 \phi(x) + \lambda_1 \phi'(x)^2 + \lambda_2 x \phi'(x)}{\alpha} + \ln |\alpha +  \phi''(x)g|\right\}\right]  \Bigg\}. \nonumber
\end{align}
This can be further simplified, as the extrema over $A,t$ are trivially solved and give the value of $\lambda_2 = \mathrm{Re}[g]$ and $\lambda_1 = (2 A)^{-1}$. Thus we obtain:
\begin{align}\label{eq:annealed_no_signal_reduced}
\lim_{n \to \infty} \frac{1}{n}&\ln \EE \ \mathrm{Crit}_{n,L_1}(B) = \\ 
&\sup_{l \in B} \extr_{\substack{\{\lambda_0,\lambda_1\} \\ g \in \bbC_+}}  \Bigg\{K(\alpha) + \lambda_0 l + \frac{1+\ln 2}{2} + \frac{1}{2} \ln \lambda_1   - \ln |g|  + \epsilon \mathrm{Im}[g]  \nonumber \\ 
 & + \alpha \ln \left[\int_\bbR {\cal D}x \exp\left\{- \frac{\lambda_0 \phi(x) + \lambda_1 \phi'(x)^2 + \mathrm{Re}[g] x \phi'(x)}{\alpha} + \ln |\alpha +  \phi''(x)g|\right\}\right]  \Bigg\}. \nonumber
\end{align}
Let us now denote
\begin{align*}
\langle\cdots\rangle_{\lambda_0,\lambda_1,g} \equiv \frac{\int_\bbR {\cal D}x (\cdots)\exp\left\{- \alpha^{-1}\left[\lambda_0 \phi(x) + \lambda_1 \phi'(x)^2 + \mathrm{Re}[g] x \phi'(x)\right] + \ln |\alpha +  \phi''(x)g|\right\}}{\int_\bbR {\cal D}x \exp\left\{- \alpha^{-1}\left[\lambda_0 \phi(x) + \lambda_1 \phi'(x)^2 + \mathrm{Re}[g] x \phi'(x)\right] +  \ln |\alpha +  \phi''(x)g|\right\}},
\end{align*}
then the fixed point equations of eq.~\eqref{eq:annealed_no_signal_reduced} can be written as:
\begin{subnumcases}{\label{eq:fixed_point_annealed_final}}
l = \langle \phi(x)\rangle_{\lambda_0,\lambda_1,g}, & \\
\frac{1}{2 \lambda_1}  = \langle\phi'(x)^2\rangle_{\lambda_0,\lambda_1,g}, & \\
-\frac{\mathrm{Re}[g]}{|g|^2} = \Big\langle x \phi'(x) - \frac{\alpha \phi''(x) (\alpha + \phi''(x)\mathrm{Re}[g])}{|\alpha + \phi''(x)g|^2}\Big\rangle_{\lambda_0,\lambda_1,g}, & \\
\epsilon - \frac{\mathrm{Im}[g]}{|g|^2} = - \Big\langle \frac{\alpha \phi''(x)^2 \mathrm{Im}[g]}{|\alpha + \phi''(x)g|^2} \Big\rangle_{\lambda_0,\lambda_1,g}. &
\end{subnumcases}
These equations are to be iterated over $\lambda_0,\lambda_1,g$, and $l$ (while enforcing the constraint $ l \in B$). 
From experience, the best procedure is to start from the solution of the unconstrained problem (without any constraint on the loss value), before smoothly following 
the solution while adding the constraint. In the case of $L_2(\bx)$, one would follow a similar procedure.
\section{The quenched complexity and the replica method}\label{sec:quenched}

In this section we detail the principle of the quenched calculation that gives rise to Results~\ref{result:quenched_nosignal}-\ref{result:quenched_signal}.
For the sake of the presentation we restrict to Result~\ref{result:quenched_nosignal}, while Result~\ref{thm:annealed_signal} will be discussed in Appendix~\ref{subsec_app:L2}. We therefore focus on the function $L_1$ of eq.~\eqref{eq:model_nosignal}. As the very basis of this calculation is non-rigorous we present this calculation in a fashion closer to theoretical physics standards, differently from Section~\ref{sec:annealed} in which we present rigorous results on the annealed complexity. Some technicalities will be postponed to Appendix~\ref{sec:app_quenched}.

\subsection{The replica trick and the \texorpdfstring{$p$}{p}-th moment}

The replica method is a heuristic tool of theoretical physics that allows to compute the quenched values of observables in the thermodynamic (i.e. $n\rightarrow \infty)$ limit  from the knowledge of their integer moments, under some assumptions. It is based on the non-rigorous identity (note that it involves an inversion of limits), for a strictly positive function $f$ of a $n$-dimensional random vector $\bx$:
\begin{align}\label{eq:replicas}
    \lim_{n \to \infty} \EE \ln f(\bx) &= \lim_{p \to 0^+} \lim_{n \to \infty} \frac{1}{p} \ln \left[\EE f(\bx)^p\right].
\end{align}
\cite{mezard1987spin} gives a comprehensive introduction to the replica method and its (many) physical insights and consequences.
Let $B \subseteq \bbR$ an open interval. The Kac-Rice formula can be stated for the $p$-th moment of the complexity \citep{azais_wschebor,adler2009random}:
\begin{align}
   \EE \, \mathrm{Crit}_{n,L_1}(B)^p &= \left[\prod_{a=1}^p \int_{\mathbb{S}^{n-1}} \mu_{\mathbb{S}}(\mathrm{d}\bx^a)\right] \mathds{1}\left[\{L_1(\bx^a) \in B\}_{a=1}^p\right] \, \varphi_{\{\mathrm{grad} \, L_1(\bx^a)\}_{a=1}^p}(0) \, \nonumber \\ 
   & \times \EE \left[ \prod_{a=1}^p \left|\det \mathrm{Hess} \,L_1(\bx^a)\right| \middle| \{\mathrm{grad} \, L_1(\bx^a)\}_{a=1}^p = 0\right].
\end{align}
Here, $\varphi_{\{\mathrm{grad} \, L(\bx^a)\}_{a=1}^p}(0)$ represents the joint density of the $p$ gradients, taken at $0$.
Note that the non-linearity $L_1(\bx)$ only depends on the parameters $y^a_\mu \equiv \bxi_\mu \cdot \bx^a $, so we will often write $ L_1(\by) \equiv L_1(\bx)$.
Proceeding as in the annealed case, we can rewrite the expectations by conditioning over $\{\by^a\}_{a=1}^p$:
\begin{align}\label{eq:KR_replicated_conditionned}
   \EE \, \mathrm{Crit}_{n,L_1}(B)^p &= \left[\prod_{a=1}^p \int \mu_{\mathbb{S}}(\mathrm{d}\bx^a)\right] \EE_{\{\by^a\}} \left[\mathds{1}\left[\{L_1(\by^a) \in B\}_{a=1}^p\right] \, \varphi_{\{\mathrm{grad} \, L_1(\bx^a)\}_{a=1}^p \Big | \{\by^a\}}(0) \, \right. \nonumber \\ 
   & \left. \EE \left[ \prod_{a=1}^p \left|\det \mathrm{Hess} \,L_1(\bx^a)\right|\middle| \{\mathrm{grad} \, L_1(\bx^a) = 0, \by^a\}_{a=1}^p\right] \right].
\end{align}
The gradient and Hessian at $\bx^a$ live in the tangent plane to the sphere at $\bx^a$, identified with $\bbR^{n-1}$.
Note that the $\{y^a_\mu\}$ are Gaussian variables with zero mean and covariance $\EE [y^a_\mu y^b_\nu] = \delta_{\mu \nu} q_{ab}$,
with $q_{ab} \equiv \bx^a \cdot \bx^b$ the ``overlap'' between replicas $a$ and $b$. 
We introduce the variables $\{q_{ab}\}$ via delta functions in eq.~\eqref{eq:KR_replicated_conditionned}:
\begin{align}\label{eq:KR_replicated_overlaps}
   \EE \ \mathrm{Crit}_{n,L_1}(B)^p &= \left[\prod_{a=1}^p \int \mu_{\mathbb{S}}(\mathrm{d}\bx^a)\right] \left[\prod_{a < b} \int \mathrm{d}q_{ab}\delta(q_{ab} - \bx^a \cdot \bx^b)\right] \EE_{\{\by^a\}} \Bigg\{\mathds{1}\left[\{L_1(\by^a) \in B\}_{a=1}^p\right] \,  \, \nonumber \\ 
   & \hspace{-1.5cm} \varphi_{\{\mathrm{grad} \, L_1(\bx^a)\}_{a=1}^p \Big | \{\by^a\}}(0) \EE \left[ \prod_{a=1}^p \left|\det \mathrm{Hess} \,L_1(\bx^a)\right|\middle| \{\mathrm{grad} \, L_1(\bx^a) = 0, \by^a\}_{a=1}^p \right] \Bigg\}.
\end{align}
 Since we fixed the $\{q_{ab}\}$, the distribution of the $\{\by^a\}$ is fixed, as well as the joint distribution of the loss, gradients and Hessians, as we will explicit in the following. As the number of overlap variables is $p(p-1)/2  = {\cal O}_n(1)$, we will perform a saddle-point over the variables $\{q_{ab}\}$ in the thermodynamic limit.
The replica-symmetric assumption (see \cite{mezard1987spin}) is a crucial hypothesis that can be made in the framework of the replica method. It amounts to assume that, once the saddle-point is performed, the extremizing $\{q_{ab}\}$ are ``symmetric'' over the different replicas of the system. Concretely, we assume that the variables $\{q_{ab}\}$ satisfy $q_{aa} = 1$, $q_{ab} = q$  for $a \neq b$. Assuming this structure of the overlap matrix allows to extend the expression of the moments to arbitrary non-integer $p$, and then to take the $p \to 0^+$ limit as needed in eq.~\eqref{eq:replicas}. We used a replica symmetric structure, which is correct in many cases, and a very good approximation in others were replica symmetry has to be broken.

\subsubsection{The phase volume factor}\label{subsubsec_main:quenched_phase_volume}
 
Let us first compute the phase space factor in eq.~\eqref{eq:KR_replicated_overlaps}. More precisely, the term:
\begin{align*}
  \left[\prod_{a=1}^p \int \mu_{\mathbb{S}}(\mathrm{d}\bx^a)\right] \left[\prod_{a < b} \delta(q_{ab} - \bx^a \cdot \bx^b)\right] &=  n^{-\frac{p(p-1)}{2}}\left[\prod_{a=1}^p \int_{\bbR^n} \mathrm{d}\bx^a \prod_{a \leq b} \delta\left(n q_{ab} - n  \bx^a \cdot \bx^b\right)\right],
\end{align*}
in which we denoted $q_{aa} = 1$. 
As we detail in Appendix~\ref{subsec_app:quenched_phasevolume} we reach, when $p \to 0^+$ and $n \to \infty$:
\begin{align}\label{eq:result_phaseterm_replicas}
   \frac{1}{np} \ln \left[\prod_{a=1}^p \int_{\bbR^n} \mathrm{d}\bx^a \prod_{a \leq b} \delta\left(n q_{ab} - n  \bx^a \cdot \bx^b\right)\right] &\simeq \frac{1}{2} \log \frac{2 \pi}{n} + \frac{1}{2} \left[\frac{1}{1-q} + \log (1-q)\right].
\end{align}
  
\subsubsection{The joint density of the gradients}

We will now compute the joint density of the gradients at $\{\bx^a\}$, conditioned on the values of $\{\by^a\}$.
The calculation is an extension of Sections~V.C and V.E of \cite{ros2019complex}. We consider two vectors $\bx^a$ and $\bx^b$ of overlap $q_{ab} = q$. It is easy to see that $\EE [\mathrm{grad} \, L(\bx^a) | \{\by^b\}_{b=1}^p] = 0$ from eq.~\eqref{eq:dist_grad_L1}, so we will focus on the covariance matrix
$\EE [\mathrm{grad} \, L(\bx^a)  \mathrm{grad} \, L(\bx^b)^\intercal | \{\by^c\}_{c=1}^p]$. 
After some calculations detailed in Appendix~\ref{subsec_app:quenched_gradient_density} we get the gradient density at leading exponential order:
\begin{align}\label{eq:gradient_density}
   \varphi_{\{\mathrm{grad} \, L_1(\bx^a)\}_{a=1}^p  | \{\by^a\}}(0) &\simeq \prod_{a \neq b} \delta \left[\frac{1}{m} \sum_{\mu=1}^m \phi'(y^a_\mu) \left(z_p(q) y^a_\mu + f_p^0(q) y^b_\mu + f_p(q)\sum_{c (\neq a,b)} y^c_\mu \right)\right] \nonumber\\
   & \hspace{-0.5cm}\times \exp\left\{\frac{np}{2} \log \frac{m}{2\pi} - \frac{n}{2} \ln \det \left[\left(\frac{1}{m} \sum_{\mu=1}^m \phi'(y^a_\mu) \phi'(y^b_\mu)\right)_{1 \leq a,b \leq p}\right] \right\}, 
\end{align}
in which the auxiliary functions $(z_p(q),f_p(q),f_p^0(q))$ are defined in eq.~\eqref{eq:def_auxiliary_functions}. 

\subsubsection{Factorization of the mean product of determinants}

The argument of this section is very close to Section~V.F of \cite{ros2019complex}.
We consider the term:
\begin{align}
\EE \left[ \prod_{a=1}^p \left|\det  \mathrm{Hess} \,L_1(\bx^a)\right|\middle| \{\mathrm{grad} \, L_1(\bx^a) = 0, \by^a\}_{a=1}^p\right] .
\end{align}
We make two important remarks, which are straightforward transpositions of the arguments of \cite{ros2019complex} to our problem, and we refer to this work for more extensive physical justifications.
\begin{itemize}
   \item The conditioning over the gradients being zero, as in the annealed calculation, only gives a finite-rank change to the Hessians $\mathrm{Hess}\, L_1(\bx^a)$ and thus 
   does not modify the limit at the scale $e^{\Theta(n)}$. At this scale, the statistics of the $p$ matrices $\{\mathrm{Hess} \, L_1(\bx^a)\}_{a=1}^p$ are identical.
   \item The spectral measure of $\mathrm{Hess}\, L_1(\bx^a)$ concentrates at a rate at least $n^{1+\epsilon}$ for a small enough $\epsilon >0$ (we expect that the actual rate is $n^2$), so that at the order $e^{\Theta(n)}$ the expectation 
   value factorizes and we can assume all the Hessians to be independent. 
\end{itemize}
Before stating the consequences of such remarks, we give some definitions:
\begin{itemize}
  \item $\mu_{\mathrm{G},q}$ is the Gaussian probability measure on $\bbR^p$ with zero mean and covariance $\EE [X_a X_b] = (1-q)\delta_{ab} + q$. Note that $\{\by_\mu\}_{\mu=1}^m$ are i.i.d.\ variables distributed according to $\mu_{\mathrm{G},q}$.
  \item We define $\nu_\by$ as the empirical measure of $(\by_1,\cdots,\by_m)$, that is $\nu_\by \equiv \frac{1}{m} \sum_{\mu} \delta_{\by_\mu}$. For every $a$, we denote $\nu^a_\by$ its marginal distribution: $\nu^a_\by(\mathrm{d}\lambda^a) \equiv \int \prod_{b (\neq a)} \nu_\by(\mathrm{d}\lambda)$. Then
  $\nu^a_\by$ is also the empirical distribution of $(y^a_\mu)_{\mu=1}^m$.
\end{itemize}
Our remarks show that we can use the results of the annealed calculation, and we have here by factorization of the expectation of the determinants, at leading exponential order:
\begin{align}\label{eq:hessian_term}
\EE \left[ \prod_{a=1}^p \left|\det \mathrm{Hess} \,L_1(\bx^a)\right|\middle| \{\mathrm{grad} \, L_1(\bx^a)\}_{a=1}^p = 0, \, \{\by^a\}\right] & \simeq e^{n \sum_{a=1}^p \kappa_{\alpha,\phi}\left(\nu^a_\by,t_\phi(\nu^a_\by)\right)}.
\end{align}
  
\subsection{Decoupling the replicas and the \texorpdfstring{$p \to 0^+$}{p->0} limit}

We can then apply Sanov's theorem to the empirical measure $\nu_\by \in \mathcal{M}(\bbR^p)$.
Recall that we have constraints on this measure by the density of the gradient and the fixation of the energy level. More precisely, we denote $\mathcal{M}^{(p)}_\phi(q,B)$ the set of probability measures on $\bbR^p$ that satisfy the following:
\begin{subnumcases}{\label{eq:conditions_Mpq}}
\label{eq:conditions_Mpq_1}
  \forall 1\leq a \leq p, \, \int \nu(\mathrm{d}\lambda) \, \phi(\lambda^a) \in B, & \\
  \label{eq:conditions_Mpq_2}
  \forall 1 \leq a \neq b \leq p, \, \int \nu(\mathrm{d}\lambda) \, \phi'(\lambda^a) \left[z_p(q) \lambda^a + f_p^0(q) \lambda^b + f_p(q) \sum_{c(\neq a,b)} \lambda^c\right] = 0. &
\end{subnumcases}
Recall that the functions $(z_p(q),f_p^0(q),f_p(q))$ are defined in eq.~\eqref{eq:def_auxiliary_functions}.
Leveraging from the results of eqs.~\eqref{eq:result_phaseterm_replicas}, \eqref{eq:gradient_density} and \eqref{eq:hessian_term}, we obtain from Sanov's theorem and Varadhan's lemma:
\begin{align}\label{eq:p_moment}
  &\lim_{n \to \infty}\frac{1}{n} \ln \EE \,\left[\mathrm{Crit}_{n,L_1}(B)^p\right] = \frac{p}{2} \ln \alpha + \sup_{q \in (0,1)} \sup_{\nu \in {\cal M}_\phi^{(p)}(q,B)} \Bigg[\frac{p}{2} \left(\frac{1}{1-q} + \ln(1-q)\right)  \\
  &  -\frac{1}{2} \ln \det \left[\left(\int \nu(\mathrm{d}\lambda) \phi'(\lambda^a) \phi'(\lambda^b)\right)_{1 \leq a,b \leq p}\right] +  \sum_{a=1}^p \kappa_{\alpha,\phi}\left(\nu^a,t_\phi(\nu^a)\right) - \alpha H(\nu | \mu_{\mathrm{G},q})\Bigg]. \nonumber
\end{align} 
Recall that $\nu^a$ is the marginal distribution of $\nu$ for the variable $\lambda^a$. 
We can then decouple the replicas under an assumption on the measure $\nu$ that amounts for replica symmetry. 
We stress that this replica symmetric assumption in the Kac-Rice calculation actually corresponds to a 1-step replica symmetry breaking (1RSB)
structure of the zero-temperature Gibbs measure, that is an exponential number of single-point metastable states that all have the same two-point overlap.
While possibly not exact, this assumption should already yield a good approximation to the landscape, and could be analytically checked by studying the stability of the replica-symmetric 
ansatz within replica theory.
This allows to take subsequently the $p \to 0^+$ limit, and after some simplifications, we reach from eq.~\eqref{eq:p_moment} the expression of Result~\ref{result:quenched_nosignal}. These steps are fairly technical, and are postponed to Appendix~\ref{subsec_app:quenched_decoupling}.


\section*{Acknowledgments}

We thank for interesting discussions Chiara Cammarota, Florent Krzakala, Valentina Ros, Lenka Zdeborov\'a. 
We are also grateful to Daniel Fletcher and Antonio Auffinger for their feedback on the proof, which led to minor corrections and improvements.
GB acknowledges support by the Simons Foundation collaboration Cracking the Glass Problem (No. 454935 to G. Biroli). Additional funding is acknowledged by AM from ``Chaire de recherche sur les mod\`eles et sciences des donn\'ees'', Fondation CFM pour la Recherche-ENS.

\bibliographystyle{plainnat}
\bibliography{refs}

\changelocaltocdepth{1}
\appendix
\newpage
\section{Technical steps of the proof of Theorem~\ref{thm:annealed_nosignal}}\label{sec_app:technical}

\subsection{Proof of Lemma~\ref{lemma:Kac_Rice}}\label{subsec_app:proof_lemma_KR}

We will apply the Kac-Rice machinery in the form of the remark made in Paragraph~6.1.4 of \cite{azais_wschebor}. We recall it as a theorem: 
\begin{theorem}[Azais-Wschebor]\label{thm_app:aw}
Let $k,d \in \bbN^\star$. Let $Z : U \to \bbR^d$ be a random field, in which $U$ is an open subset of $\bbR^d$. Assume that for every $t \in U$, we can write $Z(t) = H\left[Y(t)\right]$, such that:
\begin{enumerate}[label=(\roman*)]
\item $\{Y(t), t \in U\}$ is a Gaussian random field with values in $\bbR^k$, $\mathcal{C}^1$ paths, and such that for every $t \in U$, the distribution of $\,Y(t)$ is non-degenerate. 
\item $H : \bbR^k \to \bbR^d$ is a $\mathcal{C}^1$ function.
\item For all $t \in U$, $Z(t)$ has a density $\varphi_{Z(t)}(x)$, which is a continuous function of $(t,x) \in U \times \bbR^d$. 
\item $\mathbb{P}\left[\exists \, t \in U \, \mathrm{s.t.} \, Z(t) = 0 \text{ and } \det\left[\nabla Z(t)\right] = 0\right] = 0$.
\end{enumerate}
Define, for every compact set $B \subseteq U$,  $N(Z,B)$ to be the (finite) number of zeros of $Z$ in $B$. Then:
\begin{align}
\EE\left[N(Z,B)\right] &= \int_B \EE \left[\left|\det \nabla Z(t)\right| \Big| Z(t) = 0 \right]\varphi_{Z(t)}(0)\mathrm{d}t.
\end{align}
\end{theorem}
We wish to apply this theorem to the gradient $\mathrm{grad} L_1(\bx)$. Verifying its hypotheses will end the proof of Lemma~\ref{lemma:Kac_Rice}. We denote $\bxi \in \bbR^{n \times m}$ the matrix $\{\xi_{i \mu}\} = \{(\bxi_\mu)_i\}$, $\nabla L_1$ the Euclidean gradient of $L_1$, and $P_\bx^\perp$ the orthogonal projection on $T_\bx \mathbb{S}^{n-1}$. Since $\mathrm{grad} L_1(\bx) = P_\bx^\perp \nabla L_1(\bx)$ we have:
\begin{align}\label{eq:gradient_L1}
\mathrm{grad} \, L_1(\bx) &= \frac{1}{m} \sum_{\mu=1}^m \left(P_\bx^\perp \bxi_\mu\right) \phi' \left(\bxi_\mu \cdot \bx\right).
\end{align}
We will apply Theorem~\ref{thm_app:aw} with $d = n-1$ and $k = m \times n$. The Gaussian random field $Y(\bx) \in \bbR^{n \times m}$ is defined as $Y(\bx) \equiv \begin{pmatrix}
P_\bx^\perp \bxi_1 & \cdots & P_\bx^\perp \bxi_m \\
\bxi_1 \cdot \bx & \cdots & \bxi_m \cdot \bx
\end{pmatrix}$. Since $Y(\bx)$ is just $\bxi$ written in an orthonormal basis of $\bbR^n$ whose first vector is $\bx$, its distributions is non-degenerate. $H : \bbR^{n \times m} \to \bbR^{n-1}$ is defined as:
\begin{align}
\forall 1 \leq i < n, \hspace{0.3cm} H(Y)_i &\equiv \frac{1}{m} \sum_{\mu=1}^m Y_{i,\mu} \phi'(Y_{n,\mu}), \hspace{0.3cm}  (Y \in \bbR^{n \times m}).
\end{align}
Since $\phi$ is $\mathcal{C}^2$, $H$ is $\mathcal{C}^1$. This verifies $(i)$ and $(ii)$. We turn our attention to verifying $(iii)$. One can write the distribution of the gradient of eq.~\eqref{eq:gradient_L1} as 
$\mathrm{grad} \, L_1(\bx) \overset{d}{=} (1/m) \sum_{\mu=1}^m \phi' \left(y_\mu\right) \bz_\mu$,
in which $y_\mu \overset{\mathrm{i.i.d.}}{\sim} \mathcal{N}(0,1)$, $\bz_\mu \overset{\mathrm{i.i.d.}}{\sim} \mathcal{N}(0,I_{n-1})$, and all $\{y_\mu,\bz_\nu\}$ are independent. Since the distribution of $\mathrm{grad} \, L_1(\bx)$ does not depend on $\bx$, it is enough to check that its density exists and is a continuous function. To do so, we will show that its characteristic function $\hat{\varphi}_{\mathrm{grad} L_1(\bx)} \in L^1(\bbR^{n-1})$.
We denote $\hat{\varphi}_a$ the characteristic function of the random variable $a \equiv \phi'(y)$, and one obtains:
\begin{align*}
\norm{\hat{\varphi}_{\mathrm{grad} L_1(\bx)}}_1 &= \int_{\bbR^{n-1}} \mathrm{d}\bt \left|\EE_{\bz \sim \mathcal{N}(0,I_{n-1})} \, \hat{\varphi}_a \left(\frac{\bt \cdot \bz}{m}\right)\right|^m = \int_{\bbR^{n-1}} \mathrm{d}\bt \left|\EE_{z \sim {\cal N}(0,1)} \, \hat{\varphi}_a \left(\frac{\norm{\bt} z}{m}\right)\right|^m, \\
&= \frac{2 \pi^{\frac{n-1}{2}} m^{n-1}}{\Gamma\left(\frac{n-1}{2}\right)}\int_0^\infty \mathrm{d}q \,  q^{n-2} \left|\EE_{z} \, \hat{\varphi}_a \left(q z\right)\right|^m.
\end{align*}
Since $\alpha > 1$, if $\left[q \EE_{z} \, \hat{\varphi}_a \left(q z\right) \right] = \mathcal{O}_{q \to \infty} (1)$ we can conclude that $\norm{\hat{\varphi}_{\mathrm{grad} L_1(\bx)}}_1 < \infty$. And:
\begin{align*}
q \ \EE_{z} \, \hat{\varphi}_a \left(q z\right) &= \int_\bbR \frac{ \mathrm{d}z}{\sqrt{2 \pi}} e^{-\frac{z^2}{2 q^2}} \hat{\varphi}_a(z) =  \int_\bbR \frac{ \mathrm{d}z}{\sqrt{2 \pi}} \EE \left[e^{-\frac{z^2}{2 q^2}} e^{i a z}\right] = \frac{1}{q} \EE \left[ e^{-\frac{q^2 a^2}{2}}\right],
\end{align*}
by Fubini's theorem. Therefore $q \ \EE_{z} \, \hat{\varphi}_a \left(q z\right) \to_{q \to \infty} \varphi_a(0)$ by continuity of $\varphi_a$ around $a=0$ (Definition~\ref{def:well_behaved}), so $\norm{\hat{\varphi}_{\mathrm{grad} L_1(\bx)}}_1 < \infty$. Thus, $\mathrm{grad} L_1(\bx)$ admits the following probability density:
\begin{align}\label{eq:gradient_density_L1}
\varphi_{\mathrm{grad} L_1(\bx)}(\bu) = \frac{1}{(2 \pi)^{n-1}} \int_{\bbR^{n-1}} \mathrm{d}\bt \, e^{i \bu \cdot \bt} \left[\EE_{z} \left\{  \hat{\varphi}_a \left(\frac{\norm{\bt} z}{m}\right) \right\}\right]^m,
\end{align}
which is a continuous function of $\bu$, since  $\hat{\varphi}_{\mathrm{grad} L_1(\bx)} \in L^1(\bbR^{n-1})$. This shows $(iii)$. In order to show $(iv)$, we will use Proposition~6.5 of \cite{azais_wschebor}, that we recall here:
\begin{lemma}[Azais-Wschebor]\label{lemma:Morse}
Let $d \in \bbN^*$, and $U$ a compact subset of $\bbR^d$. Consider $Z : U \to \bbR^d$ a random field, such that $(a)$: The paths of $Z$ are of class $\mathcal{C}^2$, and $(b)$: There exists $C > 0$ such that for all $t \in U$ and all $u$ in a neighborhood of $0$, the density $\varphi_{Z(t)}$ of $Z$ verifies $\varphi_{Z(t)}(u) \leq C$.
Then $\mathbb{P}\left[\exists t \in U \text{ s.t. } Z(t) = 0 \text{ and } \det Z'(t) = 0\right] = 0$. 
\end{lemma}
Since $\phi$ is assumed to be of class $\mathcal{C}^3$, hypothesis $(a)$ is verified for $Z = \text{grad} \, L_1$. Notice then that we can fix $C > 0$ such that $\left|\EE_{z \sim \mathcal{N}(0,1)} \, \hat{\varphi}_a \left(q z\right)\right| \leq \frac{C}{1 + q}$ for all $q \geq 0$. Starting from eq.~\eqref{eq:gradient_density_L1}:
\begin{align*}
|\varphi_{\mathrm{grad} L_1(\bx)}(\bu)| \leq C_n \int_{0}^\infty \mathrm{d}q \, \frac{q^{n-2}}{(1+q)^m} \leq D_n,
\end{align*}
with $C_c, D_n$ constants depending only on $n$, using that $m \geq n$ ($\alpha > 1$). This shows $(b)$, so by Lemma~\ref{lemma:Morse}, hypothesis $(iv)$ of Theorem~\ref{thm_app:aw} follows. This ends the proof of Lemma~\ref{lemma:Kac_Rice}.

\subsection{Proof of Lemma~\ref{lemma:log_potential_L1}}\label{subsec:proof_lemma_log_pot_L1}

The proof is done in several parts, of which some are inspired by arguments of \cite{silverstein1995strong,silverstein1995empirical,silverstein1995analysis,bai2010spectral}.

\subsubsection{Technicalities on the Hessian}\label{subsubsec:low_rank}

We begin by a quick lemma on $\Lambda(\by)$, defined in eq.~\eqref{eq:def_Lambda}. 
\begin{lemma}[Low-rank perturbation]\label{lemma:lambda}
     Since the distributions of $\bz$ and $\by$ are independent, by rotation invariance we can assume that $\Lambda(\by)$ is a diagonal matrix with elements $\Lambda_\mu(\by)$. There exists a constant, denoted $||D||_\infty$, such that for all $n,y$, $|D(y)| \leq ||D||_\infty$. Then we have:
     \begin{itemize}
         \item[$(i)$] $\sup_{\by \in \bbR^m} \sup_{1\leq \mu \leq m}|\Lambda_\mu(\by)| \leq 4 ||D||_\infty$.
         \item[$(ii)$] Let $\bZ \in \bbR^{(n-1) \times m}$ be i.i.d.\ variables with zero mean and unit variance. We denote $\mu_D^{(n)}$ and $\mu_\Lambda^{(n)}$ the empirical eigenvalue distributions of $\frac{1}{n} \bZ D(\by) \bZ^\intercal$ and $\frac{1}{n} \bZ \Lambda(\by) \bZ^\intercal$ respectively. Then for all $\eta \in (0,1)$, $\left\{n^\eta \EE_\bz \left[\mu_D^{(n)} - \mu_\Lambda^{(n)}\right] \right\} \to_{n \to \infty} 0$ weakly and uniformly in $\by \in \bbR^m$.
     \end{itemize}
\end{lemma}
\begin{proof}
Recall that $|D(y)| = (n/m) |\phi''(y)| $. Since $m/n \to \alpha > 1$ and $\phi''$ is bounded, $|D(y)|$ is bounded (uniformly over $n,\by$) by a constant that we denote $||D||_\infty$.
Note that $\sup_{1\leq \mu \leq m}|\Lambda_\mu(\by)| = \sup_{||\bu||=1} \bu^\intercal \Lambda(\by) \bu$. Using eq.~\eqref{eq:def_Lambda} and denoting $\bv(\by) \equiv \phi'(\by)/||\phi'(\by)||$, we reach 
\begin{align*}
    \sup_{||\bu||=1} \bu^\intercal \Lambda(\by) \bu &\leq ||D||_\infty + \sup_{||\bu||=1}\left[|\bv^\intercal D \bv| (\bu^\intercal \bv)^2 + 2 (\bu^\intercal \bv) |\bv^\intercal D \bu|\right], \\
    &\leq 2||D||_\infty + 2 \sup_{||\bu||=1}\left[(\bu^\intercal \bv) |\bv^\intercal D \bu|\right] \leq 4 |D||_\infty,
\end{align*}
in which we used the uniform boundedness of $|D(y_\mu)|$, and the Cauchy-Schwarz inequality. This proves $(i)$. We note that $\frac{1}{n} \bz \Lambda(\by) \bz^\intercal$ and $\frac{1}{n} \bz D(\by) \bz^\intercal$ differ by a rank-$2$ matrix. $(ii)$ is thus an immediate application of the following lemma (from a course of C. Bordenave):
\begin{lemma}\label{lemma:deviation_inequality}
    Let $n \geq 1$, and $A,B$ two symmetric matrices of size $n$, such that the rank of $A-B$ is $r$. Denote $F_A$ (resp.\ $F_B$) the c.d.f.\ of the empirical spectral distribution of $A$ (resp.\ $B$). Then 
    \begin{align*}
        \sup_{t \in \bbR}|F_A(t) - F_B(t)| &\leq \frac{r}{n} \ .
    \end{align*}
\end{lemma}
\noindent
This ends the proof of Lemma~\ref{lemma:lambda}.
\end{proof}
\begin{proof}[Lemma~\ref{lemma:deviation_inequality}]
We note $\lambda_1(A) \geq \cdots \geq \lambda_n(A)$ the eigenvalues of $A$ (and similarly for $B$).
Recall weak Weyl's interlacing inequalities \citep{weyl1912asymptotische}: for every $1 \leq i \leq n$, $\lambda_{i+r}(A) \leq \lambda_i(B) \leq \lambda_{i-r}(A)$ (we use the convention $\lambda_{1-i} = +\infty$ and $\lambda_{n+i} = -\infty$ for $i \geq 1$). Let $t \in \bbR$, and $i,j$ be the smallest indices such that $\lambda_i(A) \leq t$ and $\lambda_j(B) < t$. By the interlacing inequalities, $|i-j| \leq r$. And $n|F_A(t) - F_B(t)| = |(n+1-i) - (n+1-j)| \leq r$.
\end{proof}

\noindent
We have some control of the boundedness of the Hessian, summarized in two subsequent lemmas:
\begin{lemma}[Moment bound]\label{lemma:moment_bound}
  \noindent
  For all $\gamma > 0$, one has:
  \begin{align*}
    \limsup_{n \to \infty} \sup_{\by \in \bbR^m} \frac{1}{n} \ln \EE_\bz\big\{|\det H_n^\Lambda(\by)|^\gamma \big\} &< + \infty.
  \end{align*} 
\end{lemma}
\begin{proof}
  We begin by bounding the extremal eigenvalues $\lambda_\mathrm{min},\lambda_\mathrm{max}$ of $H_n^\Lambda(\by)$ (denoted $H$ for lightness):
  \begin{align*}
    \lambda_\mathrm{max} &= \sup_{\norm{\bu}^2 = 1} [\bu^\intercal H \bu] =  - \frac{1}{m} \sum_{\mu=1}^m y_\mu \phi'(y_\mu) + \sup_{\norm{\bu}^2 = 1} \Big[\frac{1}{n} \sum_{\mu=1}^m \Lambda_\mu(\by) (\bz^\intercal \bu)_\mu^2\Big], \\ 
    &\leq \norm{x \phi'(x)}_\infty + 4 \norm{D}_\infty \times  \lambda_\mathrm{max}\Big[\frac{1}{n} \bz \bz^\intercal\Big].
  \end{align*}
  We used $(i)$ of Lemma~\ref{lemma:lambda}. Note that this bound is independent of $\by$.
  In the same way we can bound $\lambda_\mathrm{min}$, and we reach:
  \begin{align*}
    \max(-\lambda_\mathrm{min},\lambda_\mathrm{max}) &\leq \norm{x \phi'(x)}_\infty + 4 \norm{D}_\infty \times  \lambda_\mathrm{max}\Big[\frac{1}{n} \bz \bz^\intercal\Big], 
  \end{align*}
  Using this identity, we have the bound:
  \begin{align}\label{eq:bound_detH_gamma_wishart}
    \frac{1}{n} \ln \EE_\bz\big\{|\det H_n^\Lambda(\by)|^\gamma \big\} \leq \frac{1}{n} \ln \EE_\bz \exp \Big\{n \gamma \ln \Big(\norm{x \phi'(x)}_\infty + 4 \norm{D}_\infty \lambda_\mathrm{max}(\bz \bz^\intercal / n)\Big)\Big\}
  \end{align}
  It is then a very classical result of random matrix theory that the largest eigenvalue of a Wishart matrix $\bz \bz^\intercal / n$ satisfies a large deviation principle in the scale $n$.
  This is stated e.g.\ in Theorem~2.4 of \cite{biroli2020large},which gives moreover the behavior of the 
  rate function $I(x)$.
  We recall some of its properties:
  \begin{itemize}[leftmargin=*]
      \item $I(x) = +\infty$ if $x < s_\mathrm{max} \equiv (1 + \alpha^{-1/2})^2$.
      \item $I(x): [s_\mathrm{max},+\infty) \to \bbR_+$ is continuous and increasing.
      \item $I(x) \sim_{x \to \infty} x / 2$.
  \end{itemize}
  In particular, by Varadhan's lemma it implies that for any $C,D,\gamma > 0$:
  \begin{align*}
    \limsup_{n \to \infty} \frac{1}{n} \ln \EE_\bz \exp \Big\{n \gamma \ln \Big[C + D\lambda_\mathrm{max}(\bz \bz^\intercal / n)\Big]\Big\} < +\infty.
  \end{align*}
  Combining this inequality with eq.~\eqref{eq:bound_detH_gamma_wishart} ends the proof of Lemma~\ref{lemma:moment_bound}.
\end{proof}
\begin{lemma}[Properties of $\mu_{\alpha,\phi}$]\label{lemma:log_pot_existence}
    Denote $\rho_{n}(\by)$ the spectral radius of $H_n^\Lambda(\by)$. There exists $C > 0$ such that:
    \begin{itemize}
        \item[$(i)$] With probability $1$, $\limsup_{n \to \infty} \sup_{\by \in \bbR^m} \rho_{n}(\by)< C$.
        \item[$(ii)$]The support of $\mu_{\alpha,\phi}[\nu_\by^m]$ is included in $(-C,C)$ uniformly over $\by$ and $n$. 
        \item[$(iii)$] For all $\by \in \bbR^m$, $\mu_{\alpha,\phi}[\nu_\by^m]$ has a well-defined and continuous density outside $x=0$.
    \end{itemize}
\end{lemma}
\noindent
Points $(ii)$ and $(iii)$ of Lemma~\ref{lemma:log_pot_existence} are consequences of Theorem~1.1 of \cite{silverstein1995analysis}, while item $(i)$ follows from
the boundedness of $\Lambda(\by)$ by Lemma~\ref{lemma:lambda}, and the one of $x \phi'(x)$.

\subsubsection{The cut-off and the logarithmic potential}
\label{subsubsec:cutoff}

For any $\epsilon > 0$, define $\ln_\epsilon : x \in \bbR_+^\star \mapsto  \ln\left(\max(x,\epsilon)\right)$, then $x \mapsto \ln_\epsilon |x|$ is a $\epsilon^{-1}$-Lipschitz function on $\bbR$. Let $\delta \in (0,1)$.
In this section, we show that a cut-off $\epsilon_n = n^{-\delta}$ on the smallest eigenvalues does not perturb the logarithmic potential at the thermodynamical scale. 
We rely on the following technical assumption: 
\begin{assumption}\label{hyp:technical}
For any $\delta \in (0,1)$, there exists $\eta > 0$ such that for all $t > 0$:
\begin{subnumcases}{}
    \label{eq:technical_1}
\lim_{n \to \infty} \frac{1}{n^{1+\eta}} \ln \mathbb{P} \left[\left|\frac{1}{n} \sum_{\lambda \in \mathrm{Sp}(H_n^\Lambda(\by))} \ln |\lambda| \mathds{1}\left\{|\lambda| \leq n^{-\delta}\right\} \right|\geq t\right]   = - \infty ,  & \\
    \label{eq:technical_2}
 \lim_{n \to \infty} \frac{1}{n^{1+\eta}} \ln \mathbb{P}\left[\int_{|x- t_\phi(\nu_\by^m)| \leq n^{-\delta}} \mu_{\alpha,\phi}[\nu_\by^m](\mathrm{d}x) \ \ln \left|x - t_\phi(\nu_\by^m)\right| \leq -t\right] = -\infty. &
\end{subnumcases}
\end{assumption}
Physically, this makes explicit that, with large probability, there should not be enough eigenvalues of $H_n^\Lambda(\by)$ around zero so that they contribute macroscopically to the logarithmic potential. 
While we leave it as an hypothesis, we believe that this is a consequence the natural fluctuations and repulsion of the eigenvalues of $H_n^N(\by)$. 
Similar results have been proven for different types of random matrices in \cite{arous2021exponential}, and we leave as future work a generalization of this approach 
to prove Assumption~\ref{hyp:technical}.
Denote $\{\lambda_i\}_{i=1}^{n-1}$ the (sorted) eigenvalues of $H_n^\Lambda(\by)$. We can now state:
\begin{lemma}\label{lemma:cutoff}
    There exists $\eta >0$ such that for all $K > 0$:
    \begin{align}
        \lim_{n \to \infty}\frac{1}{n^{1+\eta}}\ln \mathbb{P}\left[\left|\frac{1}{n} \ln \EE \left|\det H_n^\Lambda(\by)\right| - \frac{1}{n} \ln \EE \ e^{\sum_{i=1}^{n-1} \ln_{\epsilon_n}|\lambda_i|}\right| \geq K \right] &= - \infty.
    \end{align}
\end{lemma}
\begin{proof}
We consider $\eta$ given by eq.~\eqref{eq:technical_1}.
Let $t > 0$. We denote $A_t^{(n)}$ the event 
\begin{align}
    A_t^{(n)} &\equiv \left\{\left|\frac{1}{n} \sum_{i=1}^{n-1} \ln |\lambda_i| \mathds{1}\left\{|\lambda_i| \leq n^{-\delta}\right\} \right| \geq t \right\}.
\end{align}
We have for all $\by$ and $t > 0$ ($\bar{A}_t^{(n)}$ being the complementary event to $A_t^{(n)}$):
\begin{align*}
    \frac{1}{n} \ln \EE_\bz \ e^{\sum_{i=1}^{n-1} \ln |\lambda_i|} &\geq \frac{1}{n} \ln \EE_\bz \left[ e^{\sum_{i=1}^{n-1} \ln |\lambda_i|} \mathds{1}\left[\bar{A}_t^{(n)}\right] \right] \geq -t + \frac{1}{n} \ln \EE_\bz \left[ e^{\sum_{i=1}^{n-1} \ln_{\epsilon_n} |\lambda_i|} \mathds{1}\left[\bar{A}_t^{(n)}\right]\right].
\end{align*}
So that (using $\ln_{\epsilon_n}(x) \geq \ln(x)$ for all $x > 0$):
\begin{align*}
     0 \leq \frac{1}{n} \ln \EE \ e^{\sum_{i=1}^{n-1} \ln_{\epsilon_n}|\lambda_i|} - \frac{1}{n} \ln \EE \left|\det H_n^\Lambda(\by)\right| &\leq t - \frac{1}{n} \ln \left[1 - \frac{\EE_\bz \left[ e^{\sum_{i=1}^{n-1} \ln_{\epsilon_n} |\lambda_i|} \mathds{1}\left[A_t^{(n)}\right]\right]}{\EE_\bz \left[ e^{\sum_{i=1}^{n-1} \ln_{\epsilon_n} |\lambda_i|} \right]}\right].
\end{align*}
We know $\ln_{\epsilon_n}|x| \geq - \delta \ln(n)$.
By Lemma~\ref{lemma:moment_bound}, for all $\gamma > 0$:
\begin{align*}
 \limsup_{n \to \infty} \sup_{\by \in \bbR^m} \frac{1}{n} \ln \EE_\bz \Big[ e^{\gamma \sum_{i=1}^{n-1} \ln_{\epsilon_n} |\lambda_i|} \Big]  < +\infty.
\end{align*}
Fixing $\gamma > 1$ and using Hölder's inequality, there exists therefore $C > 0$ such that for all $K > 0$ and $t \in (0,K)$:
\begin{align*}
    &\limsup_{n \to \infty}\frac{1}{n^{1+\eta}}\ln \bbP_\by\Bigg[\Bigg|\frac{1}{n} \ln \EE_\bz |\det H_n^\Lambda(\by)| - \frac{1}{n} \ln  \EE_\bz \ e^{\sum_{i=1}^{n-1} \ln_{\epsilon_n}|\lambda_i|}\Bigg| \geq K \Bigg] \\ 
    &\leq \limsup_{n \to \infty} \frac{1}{n^{1+\eta}}\ln \bbP_\by\Bigg[\mathbb{P}_\bz[A_t^{(n)}]^{1/\gamma} \geq e^{-n(\delta \ln(n) + C)}\Big[1-e^{n(t-K)}\Big] \Bigg], \\
    &\overset{(a)}{\leq} \limsup_{n \to \infty} \frac{1}{n^{1+\eta}}\ln \Bigg\{ \frac{\bbP[A_t^{(n)}]}{ e^{-\gamma n(\delta \ln(n) + C)}[1-e^{n(t-K)}]^\gamma}\Bigg\} \overset{(b)}{=} -\infty,
\end{align*}
in which we used the Markov inequality in $(a)$ and eq.~\eqref{eq:technical_1} in $(b)$.
\end{proof}
\noindent
In the following, for simplicity we will often abusively denote $\ln_{\epsilon_n} |\det H_n^\Lambda(\by)| \equiv e^{\sum_{i=1}^{n-1} \ln_{\epsilon_n}|\lambda_i|}$ and 
$\ln_{\epsilon_n} \EE |\det H_n^\Lambda(\by)| \equiv \ln \EE \ e^{\sum_{i=1}^n \ln_{\epsilon_n}|\lambda_i|}$.

\subsubsection{Concentration of the logarithmic potential with a cut-off}\label{subsubsec:concentration_cutoff}

We show here that discarding the eigenvalues of the Hessian that are close to $0$ using a cut-off $\epsilon_n \equiv n^{-\delta}$, we have concentration of the logarithmic potential.
\begin{proposition}[Concentration of the logarithmic potential]\label{prop:concentration}
  \noindent 
Let us fix $\delta < 1/2$ and recall that $\epsilon_n = n^{-\delta}$.
  Then:
\begin{align*}
\forall t > 0, \hspace{0.15cm} \limsup_{n \to \infty} \sup_{\by \in \bbR^m} \frac{1}{n^{2(1-\delta)}} \ln \bbP_\bz\Bigg[\frac{1}{n}\Bigg| \ln_{\epsilon_n} |\det H_n^\Lambda(\by)|-\EE_\bz \ln_{\epsilon_n} |\det H_n^\Lambda(\by)|\Bigg| \geq t\Bigg] < 0.
\end{align*}
\end{proposition}
\begin{proof}[Proof of Proposition~\ref{prop:concentration}]
We will try to use traditional Lipschitz concentration bounds (cf e.g.\ \cite{anderson2010introduction}).
We will study under which conditions the function $G(\bz) \equiv (1/n) \ln_{\epsilon_n} |\det H_n^\Lambda(\by)|$
is a Lipschitz function of $\bz \in \bbR^{(n-1) \times m}$ (for fixed $\by$).
We will do it by bounding $\norm{\nabla_\bz G}_\infty$. Let $f_n(x) \equiv \ln_{\epsilon_n}|x|$ for $x \in \bbR$. We have:
\begin{align*}
\sum_{i=1}^{n-1} \sum_{\mu=1}^m \Bigg(\frac{\partial G(\bz)}{\partial z_{i \mu}}\Bigg)^2 &= \frac{1}{n^4} \sum_{i=1}^{n-1} \sum_{\mu=1}^m \Bigg[\mathrm{Tr} \Big\{f_n'\Big(\frac{1}{n} \bz \Lambda(\by) \bz^\intercal \Big) \Delta_{i \mu}\Big\}\Bigg]^2,
\end{align*}
in which $\Delta_{i \mu} \in \bbR^{(n-1) \times (n-1)}$ with $(\Delta_{i \mu})_{j k} \equiv \Lambda_\mu(\by) (\delta_{i j} z_{k \mu} + \delta_{i k} z_{j \mu})$. So one shows easily:
\begin{align}\label{eq:G_derivative}
\sum_{i=1}^{n-1} \sum_{\mu=1}^m \Bigg(\frac{\partial G(\bz)}{\partial z_{i \mu}}\Bigg)^2 &= \frac{4}{n^3} \mathrm{Tr}\Bigg[\Bigg(f_n'\Big(\frac{1}{n} \bz \Lambda(\by) \bz^\intercal \Big)\Bigg)^2 \Bigg(\frac{1}{n} \bz \Lambda(\by)^2 \bz^\intercal\Bigg)\Bigg].
\end{align}
Let us recall the Hoffman-Wielandt inequality \cite{hoffman2003variation}: 
\begin{lemma}[Hoffman-Wielandt inequality for the $L_2$ norm]\label{lemma:hoffman_wielandt}
  \noindent
Let $k \in \bbN^\star$, and $A,B \in \mathcal{S}_k(\bbR)$ be two symmetric matrices with respective eigenvalues $\lambda_1(A) \leq \cdots \leq \lambda_k(A)$ and $\lambda_1(B) \leq \cdots \leq \lambda_k(B)$. Then $\sum_{i=1}^k [\lambda_i(A) - \lambda_i(B)]^2 \leq \norm{A-B}_2^2$.
\end{lemma}
In particular if $A$ and $B$ are positive matrices one has $\mathrm{Tr} [A B] \leq \sum_i \lambda_i(A) \lambda_i(B)$.
We use this in eq.~\eqref{eq:G_derivative} along with the $n^{\delta}$-Lipschizity of $f_n$:
\begin{align}\label{eq:bound_G_Lipschitz}
\sum_{i=1}^{n-1} \sum_{\mu=1}^m \Bigg(\frac{\partial G(\bz)}{\partial z_{i \mu}}\Bigg)^2 &\leq \frac{4 n^{2\delta}}{n^4} \mathrm{Tr}\big[ \bz \Lambda(\by)^2 \bz^\intercal\big] \leq \frac{4^3 n^{2\delta} \norm{D}^2_\infty}{n^4} \sum_{\mu=1}^m \sum_{i=1}^{n-1} z_{i \mu}^2,
\end{align}
in which we used Lemma~\ref{lemma:lambda}. We denote $A$ the event 
\begin{align*}
  A \equiv \Big\{\frac{1}{n^2} \sum_{\mu=1}^m \sum_{i=1}^n z_{i \mu}^2 \geq 1 + \alpha \Big\}.
\end{align*}
It is a classical concentration result (cf e.g.\ Chapter~3.1 of \cite{vershynin2018high}) that there exists $c > 0$ such that:
\begin{align}\label{eq:bound_concentration_norm_Gaussian}
  \bbP_\bz\Big[\Big| \sqrt{\frac{1}{n^2}\sum_{\mu,i} z_{\mu i}^2} - \sqrt{\alpha}\Big| \geq t\Big] \leq 2 e^{-c n^2 t^2}.
\end{align}
In particular, this implies
\begin{align}\label{eq:bound_proba_A}
  \bbP_\bz[A] &\leq 2 e^{- c n^2 (\sqrt{1+\alpha} - \sqrt{\alpha})^2}.
\end{align}
Let us now show that it suffices to prove the bound of Proposition~\ref{prop:concentration} assuming that $A$ does not occur.
Indeed, $n^{-2 (1-\delta)} \ln \bbP_\bz[A] \leq - c n^{2 \delta}$ for a constant $c > 0$, and:
\begin{align*}
&\limsup_{n \to \infty} \sup_{\by \in \bbR^m} \frac{1}{n^{2(1-\delta)}} \ln \bbP_\bz\Bigg[\frac{1}{n}\Bigg| \ln_{\epsilon_n} |\det H_n^\Lambda(\by)|-\EE_\bz \ln_{\epsilon_n} |\det H_n^\Lambda(\by)|\Bigg| \geq t \Bigg| A \Bigg] \\ 
&\leq \limsup_{n \to \infty}\sup_{\by \in \bbR^m}\frac{1}{n^{2(1-\delta)}}\ln \bbP_\bz\Bigg[\Big|\det H_n^\Lambda(\by)|\Big| \geq e^{n t + \EE_\bz \ln_{\epsilon_n} |\det H_n^\Lambda(\by)|} \Bigg| A \Bigg], \\ 
&\overset{(a)}{\leq} \limsup_{n \to \infty}\sup_{\by \in \bbR^m} \Big\{\frac{1}{n^{2(1-\delta)}}\ln \EE_\bz \Big[|\det H_n^\Lambda(\by)| \Big| A \Big] - \frac{1}{n^{2(1-\delta)}} \Big(n t + \EE_\bz \ln_{\epsilon_n} |\det H_n^\Lambda(\by)|\Big) \Big\}, 
\\ &\overset{(b)}{<}+\infty.
\end{align*}
We used Markov's inequality in $(a)$. The inequality $(b)$ can be obtained by very similar arguments than the one used to prove Lemma~\ref{lemma:moment_bound}: 
by rescaling $\bz$ by $\norm{\bz}$, it is easy to see that the event $A$ will not change the scaling of the large deviations of the largest eigenvalue of the Wishart matrix $\bz \bz^\intercal / n$, 
so that the bound of Lemma~\ref{lemma:moment_bound} will also apply when conditioning by the event $A$.
All in all, 
\begin{align*}
\limsup_{n \to \infty} \sup_{\by \in \bbR^m} \Bigg\{\frac{1}{n^{2(1-\delta)}} \ln \bbP_\bz\Bigg[\frac{1}{n}\Bigg| \ln_{\epsilon_n} |\det H_n^\Lambda(\by)|-\EE_\bz \ln_{\epsilon_n} |\det H_n^\Lambda(\by)|\Bigg| \geq t \Bigg| A \Bigg] \mathbb{P}_\bz[A] \Bigg\} = - \infty.
\end{align*}
Using the law of total expectation, we can safely ignore the occurrence of $A$ when showing Proposition~\ref{prop:concentration}.
Assuming that $A$ is not occurring yields:
\begin{align}\label{eq:Lipschitz_G}
\sum_{i=1}^{n-1} \sum_{\mu=1}^m \Bigg(\frac{\partial G(\bz)}{\partial z_{i \mu}}\Bigg)^2 &\leq \frac{4^3 (1+\alpha) n^{2\delta} \norm{D}^2_\infty}{n^2}.
\end{align}
Recall the Lipschitz concentration of independent variables with laws satisfying the logarithmic Sobolev inequality with a uniform constant $c$ (see for instance \cite{anderson2010introduction} for a proof and an introduction to the logarithmic Sobolev inequalities):
\begin{lemma}[Herbst]\label{lemma:Herbst}
  \noindent
Let $n \in \bbN^\star$ and $P$ be a probability distribution on $\bbR^n$ satisfying the Logarithmic Sobolev Inequality with constant $c>0$. Let $G$ be a Lipschitz function on $\bbR^n$ with Lipschitz constant $\norm{G}_{\mathcal{L}}$. Then for all $t > 0$, 
$\bbP[|G - \EE G|\geq t] \leq 2 \exp[-t^2/(2 c \norm{G}_\mathcal{L}^2)]$.
\end{lemma}
It is easy to check that the Gaussian standard law of $\bz$, conditioned by the (extremely probable) event $\overline{A}$, satisfies the Logarithmic Sobolev Inequality with constant $c = 1 + \smallO_n(1)$.
Applying Lemma~\ref{lemma:Herbst} alongside eq.~\eqref{eq:Lipschitz_G} finishes the proof.
\end{proof}

\subsubsection{The logarithmic potential of the asymptotic measure}\label{subsubsec:asymptotic_potential}

In this part, we relate the expected logarithmic potential to the logarithmic potential of the measure $\mu_{\alpha,\phi}[\nu_\by^m]$, cf Theorem~\ref{thm:annealed_nosignal}.
\begin{proposition}[Concentration on $\kappa_{\alpha,\phi}$]\label{prop:concentration_implicit_potential}
There exists $\eta > 0$ such that for all $t > 0$:
\begin{align}\label{eq:concentration_implicit_potential}
\lim_{n \to \infty} \frac{1}{n^{1+\eta}} \ln \mathbb{P}\left[\left|\EE_\bz\frac{1}{n} \ln_{\epsilon_n} \left|\det H_n^\Lambda(\by)\right| - \kappa_{\alpha,\phi}\left(\nu_\by^m,t_\phi(\nu_\by^m)\right) \right| \geq t\right] &= -\infty.
\end{align}
\end{proposition}
\begin{proof}
The proof goes in two parts. 
 First, we show that there exists $\eta_1 > 0$ such that\footnote{Note that this result is uniform over $\by$, and thus stronger than what is needed to show Proposition~\ref{prop:concentration_implicit_potential}.}: 
\begin{align}\label{eq:concentration_implicit_potential_1}
\lim_{n \to \infty} \left[n^{\eta_1} \sup_{\by \in \bbR^m}  \left|\EE_\bz\frac{1}{n} \ln_{\epsilon_n} \left|\det H_n^\Lambda(\by)\right| - \int_\bbR \ln_{\epsilon_n} \left|x-t_{\phi}(\nu_\by^m)\right| \mu_{\alpha,\phi}[\nu_\by^m](\mathrm{d}x)\right| \right] &= 0.
\end{align}
We will then conclude by showing 
that there exists $\eta_2 > 0$ such that for all $t > 0$:
\begin{align}\label{eq:concentration_implicit_potential_2}
 \lim_{n \to \infty}\frac{1}{n^{1+\eta_2}} \mathbb{P}\left[\left| \int_\bbR \ln_{\epsilon_n} \left|x-t_{\phi}(\nu_\by^m)\right| \mu_{\alpha,\phi}[\nu_\by^m](\mathrm{d}x)-\kappa_{\alpha,\phi}\left(\nu_\by^m,t_\phi(\nu_\by^m)\right) \right| \geq t \right] &= -\infty.
\end{align}
We begin by eq.~\eqref{eq:concentration_implicit_potential_2}.
Under Assumption~\ref{hyp:technical}, we take $\eta_2$ given by eq.~\eqref{eq:technical_2}.
We have
\begin{align*}
   &\int_\bbR \ln_{\epsilon_n} \left|x-t_{\phi}(\nu_\by^m)\right| \mu_{\alpha,\phi}[\nu_\by^m](\mathrm{d}x)-\kappa_{\alpha,\phi}\left(\nu_\by^m,t_\phi(\nu_\by^m)\right) \\ 
       &= -\delta \ln(n) \mu_{\alpha,\phi}[\nu_\by^m]\left(t_\phi(\nu_\by^m)-\epsilon_n,t_\phi(\nu_\by^m)+\epsilon_n\right) -  \int_{t_\phi-\epsilon_n}^{t_\phi+\epsilon_n} \ln \left|x-t_{\phi}(\nu_\by^m)\right| \mu_{\alpha,\phi}[\nu_\by^m](\mathrm{d}x).
\end{align*}
 Since $\ln_{\epsilon}(x) \geq \ln(x)$ , we reach from this:
 \begin{align*}
 &\limsup_{n \to \infty}\frac{1}{n^{1+\eta_2}} \mathbb{P}\left[\left| \int_\bbR \ln_{\epsilon_n} \left|x-t_{\phi}(\nu_\by^m)\right| \mu_{\alpha,\phi}[\nu_\by^m](\mathrm{d}x)-\kappa_{\alpha,\phi}\left(\nu_\by^m,t_\phi(\nu_\by^m)\right) \right| \geq t \right] , \\
 &\leq \limsup_{n \to \infty}\frac{1}{n^{1+\eta_2}} \mathbb{P}\left[\int_{t_\phi-\epsilon_n}^{t_\phi+\epsilon_n} \ln \left|x-t_{\phi}(\nu_\by^m)\right| \mu_{\alpha,\phi}[\nu_\by^m](\mathrm{d}x) \leq -t \right],
\end{align*}
and using eq.~\eqref{eq:technical_2}, we reach eq.~\eqref{eq:concentration_implicit_potential_2}.
Let us show eq.~\eqref{eq:concentration_implicit_potential_1}. Its proof is based on the following lemma, a consequence of the analysis of \cite{silverstein1995empirical,bai2010spectral}:
\begin{lemma}\label{lemma:stieltjes_convergence}
    Denote $g_{n}(z)$ the Stieltjes transform of $\frac{1}{n} \bz \Lambda(\by) \bz^\intercal$, and $g_{\alpha,\phi}[\nu_\by^m](z)$ the one of $\mu_{\alpha,\phi}[\nu_\by^m]$, for $z \in \bbC_+$. Then there exists $\eta \in (0,1)$ such that for all $z \in \bbC_+$:
    \begin{align}
        \lim_{n \to \infty} \left\{ \sup_{\by \in \bbR^m} n^\eta \left|\EE_{\bz}(g_n(z)) - g_{\alpha,\phi}[\nu_\by^m](z)\right|\right\} &=0.
    \end{align}
\end{lemma}
The proof of this lemma is postponed to Appendix~\ref{subsec:proof_lemma_stieltjes}. Let us fix $\eta$ given by this lemma. As stated for instance in Theorem~2.4.4 of \cite{anderson2010introduction}, a consequence of the Stieltjes inversion theorem is that
for every Borel set $E \subseteq \bbR$, we have  
\begin{align}\label{eq:concentration_measure_mu_alpha}
  \lim_{n \to \infty}\sup_{\by \in \bbR^m}\left[n^{\eta}\left|\EE \mu_n(E) - \mu_{\alpha,\phi}[\nu_\by^m](E)\right| \right] &= 0.  
\end{align}
in which $\mu_n$ is the empirical spectral distribution of $\frac{1}{n} \bz \Lambda(\by) \bz^\intercal$. Fix $\eta_1 < \eta$. We have, uniformly over $\by$:
\begin{align}\label{eq:bound_log_cutoff}
    &n^{\eta_1} \left|\EE_\bz\frac{1}{n} \ln_{\epsilon_n} \left|\det H_n^\Lambda(\by)\right| - \int_\bbR \ln_{\epsilon_n} \left|x-t_{\phi}(\nu_\by^m)\right| \mu_{\alpha,\phi}[\nu_\by^m](\mathrm{d}x)\right| \nonumber \\
  &\leq n^{\eta_1} \int_{|x-t_\phi(\nu_\by^m)|>1}
 \ln \left|x-t_{\phi}(\nu_\by^m)\right| \left[\EE \mu_n - \mu_{\alpha,\phi}[\nu_\by^m]\right](\mathrm{d}x)  \\
 & \hspace{1cm}+ \delta \ln(n) n^{\eta_1} \int_{|x-t_\phi(\nu_\by^m)|<1}
 \left[\EE \mu_n - \mu_{\alpha,\phi}[\nu_\by^m]\right](\mathrm{d}x). \nonumber 
 \end{align}
Let us fix $C > 0$ given by item $(i)$ of Lemma~\ref{lemma:log_pot_existence}. Moreover, we can also bound $t_\phi(\nu_\by^m)$ by $||x\phi'||_\infty$. This gives that for $n$ large enough the quantity of eq.~\eqref{eq:bound_log_cutoff} is bounded (uniformly over $\by$) by:
 \begin{align*}
     n^{\eta_1} \left[\ln\left(C + ||x\phi'||_\infty\right) + \delta \ln(n) \right] \left\{ \left[\EE \mu_n - \mu_{\alpha,\phi}[\nu_\by^m]\right](-C,C)\right\}.
 \end{align*}
 Using eq.~\eqref{eq:concentration_measure_mu_alpha} and since $\eta_1 < \eta$, this shows eq.~\eqref{eq:concentration_implicit_potential_1}.
\end{proof}

\subsubsection{Conclusion of the proof}\label{subsubsec_app:conclusion_proof_main_lemma}

Let us conclude the proof of Lemma~\ref{lemma:log_potential_L1}
from all our results of Section~\ref{subsec:proof_lemma_log_pot_L1}.
We fix $\delta> 0$ such that $\delta < 1/2$.
Note that Proposition~\ref{prop:concentration} is a uniform result on $\by$, much stronger than what we required. 
Since it shows that the concentration (as a function of $\bz$) of the log-determinant is in a scale $n^{1+\epsilon}$ with $\epsilon > 0$,
it implies that there exists $\eta > 0$ such that for all $t > 0$:
\begin{align}
  \label{eq:concentration_potential}
    \lim_{n \to \infty} \frac{1}{n^{1+\eta}} \ln \bbP \Big[\Big|\frac{1}{n} \ln_{\epsilon_n}\EE_\bz |\det H_n^\Lambda(\by)|-\frac{1}{n}\EE_\bz \ln_{\epsilon_n} |\det H_n^\Lambda(\by)|\Big| \geq t\Big] = - \infty. & 
\end{align}
Combining this identity with Lemma~\ref{lemma:cutoff} and Proposition~\ref{prop:concentration_implicit_potential}, we reach the conclusion of Lemma~\ref{lemma:log_potential_L1}.

\subsection{Proof of Lemma~\ref{lemma:varadhan_well_posed}}\label{subsec:proof_lemma_moment_condition}

\begin{proof}
Let $\gamma \in (1,\alpha)$. We fix $C>0$ given by Lemma~\ref{lemma:log_pot_existence}. 
Then for all $\by \in \bbR^m$:
\begin{align*}
    \kappa_{\alpha,\phi}(\nu_\by^m,t_\phi(\nu_\by^m) ) 
    &\leq \int \mu_{\alpha,\phi}[\nu_\by^m](\mathrm{d}x) \ \ln (1 + |x|) + \ln (1 + |t_\phi(\nu_\by^m)|) , \\
    &\leq \ln(1+\norm{x \phi'(x)}_\infty) + \ln(1+C). 
\end{align*}
By Lemma~\ref{lemma:moment_bound} we have:
\begin{align*}
    \limsup_{n \to \infty} \sup_{\by \in \bbR^m} \Big[\frac{1}{n} \ln \EE |\det H_n^\Lambda(\by)| \Big] &< +\infty. 
\end{align*}
Therefore, in order to prove Lemma~\ref{lemma:varadhan_well_posed}, it only remains to show that 
\begin{align}\label{eq:bound_without_kappa}
\limsup_{n \to \infty} \frac{1}{n} \ln \EE_\by \Big[\exp\Big\{-\frac{\gamma n}{2}\ln \Big(\frac{1}{m} \sum_{\mu=1}^m\phi'(y_\mu)^2\Big)\Big\} \Big] < \infty.
\end{align}
Let us now prove eq.~\eqref{eq:bound_without_kappa}.
We denote $v \equiv \EE_{y \sim {\cal N}(0,1)}[\phi'(y)^2]$ and $A \equiv \norm{\phi'}_\infty^2$ .
Since $A < \infty$, we can apply Cramer's theorem to $S \equiv (1/m) \sum_{\mu}\phi'(y_\mu)^2$, so that we have:
\begin{align}\label{eq:bound_cramer_ln}
\limsup_{n \to \infty} \frac{1}{n} \ln \EE_\by \Big[\exp\Big\{-\frac{\gamma n}{2}\ln \Big(\frac{1}{m} \sum_{\mu=1}^m\phi'(y_\mu)^2\Big)\Big\} \Big]
 &\leq \sup_{S \in (0,A)} [-\frac{\gamma}{2} \ln S - \alpha \Lambda^\star(S)] ,
\end{align}
in which $\Lambda^\star(S)$ is defined as the Legendre transform of the moment generating function of $\phi'(y)^2$:
\begin{align*}
    \Lambda^\star(S) &\equiv 
    \begin{cases} 
       \sup_{\theta \geq 0} \{\theta S - \ln \EE_{y \sim {\cal N}(0,1)}[e^{\theta \phi'(y)^2}]\} \hspace{1cm} &\mathrm{if} \ S \geq v, \\ 
    \sup_{\theta \geq 0} \{-\theta S - \ln \EE_{y \sim {\cal N}(0,1)}[e^{-\theta \phi'(y)^2}]\} \hspace{1cm} &\mathrm{if} \ S \leq v.
    \end{cases}
\end{align*}
By continuity of the involved functions, in order to conclude from eq.~\eqref{eq:bound_cramer_ln} we just need to be able to show that $(i) : \limsup_{S \uparrow A} (-\Lambda^\star(S)) < \infty$ and  $(ii) : \limsup_{S \downarrow 0}[-\frac{\gamma}{2} \ln S - \alpha \Lambda^\star(S)] < \infty$.
Point $(i)$ is trivial since $\Lambda^\star(S) \geq 0$ for all $S \in (0,A)$ (it is a rate function).
To show $(ii)$, we use the fact that for all $S \in (0,v)$ and $\theta \geq 0$ we have $\Lambda^\star(S) \geq - \theta S - \ln \EE[e^{-\theta \phi'(y)^2}]$.
In particular, for $\theta = S^{-1}$ we have $\Lambda^\star(S) \geq - 1 - \ln \EE[e^{-S^{-1}\phi'(y)^2}]$. 
Since $a=\phi'(y)$ has a density $\varphi_a$ continuous around $0$ (Def.~\ref{def:well_behaved}), we fix $a_0 > 0$ such that $\varphi_a$ is continuous in $[-a_0,a_0]$. For every $\theta > 0$:
\begin{align*}
    \ln \EE [e^{-\theta \phi'(y)^2}] &\leq  \ln \Big[\EE (e^{-\theta a^2} \mathds{1}_{|a| \leq a_0}) + e^{-\theta a_0^2} \Big] \leq  \ln \Big[(\sup_{|a|\leq a_0}|\varphi_a(a)|) \frac{\sqrt{\pi}}{\sqrt{\theta}} + e^{-\theta a_0^2} \Big],
\end{align*}
and thus $\ln \EE [e^{-\theta \phi'(y)^2}] \leq C - (1/2) \ln \theta$
with a constant $C > 0$. Using this bound and the remark before we reach
\begin{align*}
    -\frac{\gamma}{2} \ln S - \Lambda^\star(S) &\leq \frac{\alpha-\gamma}{2} \ln S + \alpha (1+C).
\end{align*}
Since $\alpha-\gamma > 0$, we have $\lim_{S \downarrow 0}[-\frac{\gamma}{2} \ln S - \alpha \Lambda^\star(S)] = -\infty$, which obviously implies point $(ii)$, which in turn shows eq.~\eqref{eq:bound_without_kappa}.
\end{proof}

\subsection{Proof of \texorpdfstring{eq.~\eqref{eq:final}}{eq.(22)}}\label{subsec:proof_eq_before_Varadhan}
Let $t > 0$, and fix $\eta > 0$ given by Lemma~\ref{lemma:log_potential_L1}. We define $E_n^{(t)}$, $A_n$, $B_n$:
\begin{subnumcases}{\label{eq:def_An_Bn}}
    E_n^{(t)} \equiv \left\{\left|\frac{1}{n} \ln \EE_{\bz} \left[ \left|\det H_n^\Lambda(\by) \right| \right] - \kappa_{\alpha,\phi}\left(\nu_\by^m,t_\phi(\nu_\by^m)\right) \right| \geq t\right\}, & \\
    A_n \equiv \frac{1}{n} \ln \EE \left[\mathds{1}[L_1(\by) \in B] e^{-\frac{n}{2} \ln\left(\frac{1}{m} \sum_\mu\phi'(y_\mu)^2\right)} \EE |\det H_n^\Lambda(\by)|\right], & \\
    B_n \equiv \frac{1}{n} \ln \EE \left[\mathds{1}[L_1(\by) \in B] e^{-\frac{n}{2} \ln\left(\frac{1}{m} \sum_\mu\phi'(y_\mu)^2 \right) + n \kappa_{\alpha,\phi}(\nu_\by^m,t_\phi(\nu_\by^m))}\right]. &
\end{subnumcases}
$A_n$ is related to the complexity by Lemma~\ref{lemma:annealed_finite_n_L1}, and by Varadhan's lemma (and Lemma~\ref{lemma:varadhan_well_posed}), we have:
 \begin{align}\label{eq:limit_Bn}
     \lim_{n \to \infty} B_n = \sup_{\nu \in {\cal M}_\phi(B)} \,  \left[ - \frac{1}{2}\mathcal{E}_\phi(\nu) + \kappa_{\alpha,\phi}(\nu,t_\phi(\nu)) - \alpha H(\nu|\mu_G)\right] \in [-\infty,+\infty).
 \end{align}
The factor $\alpha$ in front of the relative entropy arises as we consider the empirical distribution of $m = \alpha n$ i.i.d.\ variables.
For all $t > 0$, we have by definition of $A_n,B_n$:
{\small
\begin{subnumcases}{}
\label{eq:lower_bound_AnBn}
    A_n - B_n \geq -t + \frac{1}{n} \ln \left[1- \frac{\EE \left[\mathds{1}_{L_1(\by) \in B ; E_n^{(t)}} e^{-\frac{n}{2} \ln\left(\frac{1}{m} \sum_\mu\phi'(y_\mu)^2 \right) + n \kappa_{\alpha,\phi}(\nu_\by^m,t_\phi(\nu_\by^m))}\right]}{\EE \left[\mathds{1}_{L_1(\by) \in B} \ e^{-\frac{n}{2} \ln\left(\frac{1}{m} \sum_\mu\phi'(y_\mu)^2 \right) + n \kappa_{\alpha,\phi}(\nu_\by^m,t_\phi(\nu_\by^m))}\right]}\right], & \\
    \label{eq:upper_bound_AnBn}
    A_n - B_n \leq t + \frac{1}{n} \ln \left[1 + e^{-nt} \frac{\EE \left[\mathds{1}_{L_1(\by) \in B ; E_n^{(t)}} e^{-\frac{n}{2} \ln\left(\frac{1}{m} \sum_\mu\phi'(y_\mu)^2 \right)}\EE |\det H_n^\Lambda(\by)|\right]}{\EE \left[\mathds{1}_{L_1(\by) \in B} e^{-\frac{n}{2} \ln\left(\frac{1}{m} \sum_\mu\phi'(y_\mu)^2 \right) }\EE |\det H_n^\Lambda(\by)|\right]}\right]. &
\end{subnumcases}
}
Using Hölder's inequality and Lemma~\ref{lemma:varadhan_well_posed}, there exists $\gamma > 1$ and a constant $C > 0$ such that:
\begin{align}
\label{eq:bound_AnBn}
     -t + \frac{1}{n} \ln \left[1 -\frac{ \mathbb{P}[E_n^{(t)}]^{\frac{1}{\gamma}}}{e^{n B_n - n C}} \right] &\leq A_n - B_n \leq t + \frac{1}{n} \ln \left[1 + \frac{\mathbb{P}[E_n^{(t)}]^{\frac{1}{\gamma}}}{e^{nt + n A_n - n C}} \right].
\end{align}
\begin{itemize}[leftmargin=*]
    \item Assume that $\lim B_n = - \infty$ and $\limsup A_n > - \infty$. Let us fix a lower-bounded sub-sequence $A_{\varphi(n)}$ of $A_n$, so that $\lim[A_{\varphi(n)} - B_{\varphi(n)}] = + \infty$. However, by eq.~\eqref{eq:bound_AnBn} and Lemma~\ref{lemma:log_potential_L1}, we have $\limsup[A_{\varphi(n)} - B_{\varphi(n)}] \leq t$, as $(1/n) \ln \mathbb{P}[E_n^{(t)}] \to - \infty$. So we showed that $\lim B_n = - \infty \Rightarrow \lim A_n = - \infty$, which shows eq.~\eqref{eq:final} in this case. 
    \item Let us now assume that $\lim B_n > - \infty$. 
    Using the left inequality of eq.~\eqref{eq:bound_AnBn} and Lemma~\ref{lemma:log_potential_L1}, we reach in the same way that $\liminf[A_n - B_n] \geq -t$, which implies that $\liminf A_n > -\infty$. Thus we can use the right inequality of eq.~\eqref{eq:bound_AnBn} to show similarly that $\limsup[A_n - B_n] \leq t$. Taking the $t \to 0$ limit finishes the proof of eq.~\eqref{eq:final}.
\end{itemize}

\subsection{Proof of Lemma~\ref{lemma:stieltjes_convergence}}\label{subsec:proof_lemma_stieltjes}

\begin{proof}
This proof is a generalization of the arguments of \cite{silverstein1995empirical}. To fix the notations here, $(\mu_n,g_n(z))$ are the ESD\footnote{Empirical Spectral Distribution.} and Stieltjes transform of $\frac{1}{n} \bz \Lambda(\by) \bz^\intercal$, $(\mu^D_n,g^D_n(z))$ the ones of $\frac{1}{n} \bz D(\by) \bz^\intercal$, and $g_{\alpha,\phi}[\nu_\by^m](z)$ is the Stieltjes transform of $\mu_{\alpha,\phi}[\nu_\by^m]$. We show:
\begin{itemize}
    \item[$(i)$] There exists $\eta_1 \in (0,1)$ such that for all $z \in \bbC_+$:
    \begin{align*}
        \lim_{n \to \infty} \left\{n^{\eta_1}\sup_{\by \in \bbR^m} \left|\EE_\bz g_n(z) - \EE_\bz g^D_n(z) \right|\right\} &= 0.
    \end{align*}
    \item[$(ii)$] There exists $\eta_2 \in (0,1)$ and $K > 0$ such that for all $z \in \bbC_+$:
    \begin{align*}
        \limsup_{n \to \infty} \left\{n^{\eta_2}\sup_{\by \in \bbR^m} \left|\EE_\bz g_n^D(z) - g_{\alpha,\phi}[\nu_\by^m](z) \right|\right\} &< K.
    \end{align*}
\end{itemize}
Points $(i)-(ii)$ obviously imply Lemma~\ref{lemma:stieltjes_convergence}, so we now prove them.
\paragraph{Proof of $(i)$} This is a direct consequence of the second part of Lemma~\ref{lemma:lambda}, which implies that we can fix $\eta_1 \in (0,1)$ such that uniformly in $\by$, $\left\{n^{\eta_1} \EE_\bz \left[\mu^D_n - \mu_n\right] \right\} \to_{n \to \infty} 0$ weakly. By a classical characterization of the Stieltjes transform (cf for instance Theorem~2.4.4 of \cite{anderson2010introduction}), this implies $(i)$.
\paragraph{Proof of $(ii)$} This part is more involved. Let us give some definitions and conventions. We define $M_n \equiv \frac{1}{n} \bz D(\by) \bz^\intercal$. If $\bz_\mu$ is the column of $\bz$ indexed by $\mu$, we denote $M_n^{(\mu)} \equiv \frac{1}{n} \sum_{\nu (\neq \mu)} D(y_\mu) \bz_\nu \bz_\nu^\intercal$, which is independent of $\bz_\mu$, and $g_n^{(\mu)}(z)$ its Stieltjes transform\footnote{To lighten the results, we state the results as if these matrices were of size $n$, even though they are of size $(n-1)$, as it does not change anything to the argument.}.
Finally, we define $x$ and $x^{(\mu)}$ as:
\begin{align}\label{eq:def_x_xmu}
    x &\equiv \frac{\alpha}{m} \sum_{\mu=1}^m \frac{\phi''(y_\mu)}{\alpha + \phi''(y_\mu) g_n(z)} \, \, , \hspace{1cm } x^{(\mu)} \equiv  \frac{\alpha}{m} \sum_{\nu (\neq \mu)} \frac{\phi''(y_\nu)}{\alpha + \phi''(y_\nu) g_n(z)}. &
\end{align}
We fix $z = z_1 + i z_2 \in \bbC_+$ and we start from the trivial identity:
\begin{align}
    -\frac{1}{z-x} \mathrm{I}_n = (M_n - z \mathrm{I}_n)^{-1} - \frac{1}{z-x} (M_n -x \mathrm{I}_n) (M_n - z \mathrm{I}_n)^{-1}.
\end{align}
For every invertible matrix $B$, vector $\bq$, and $\tau \in \bbR$, we have by the Sherman–Morrison formula:
\begin{align*}
    q^\intercal (B+ \tau q q^\intercal)^{-1} &= \frac{1}{q + \tau q^\intercal B^{-1}\tau} q^\intercal B^{-1}.
\end{align*}
Plugging it into the last equation and taking the averaged trace yields
\begin{align}\label{eq:cavity_gn}
    -\frac{1}{z - x} - g_n(z) &= -\frac{1}{z-x} \frac{1}{n} \mathrm{Tr}\left[\left(\frac{1}{n} \sum_\mu \phi''(y_\mu) \bz_\mu \bz_\mu^\intercal - x\mathrm{I}_n \right) (M_n - z \mathrm{I}_n)^{-1}\right], \nonumber \\
    &= \frac{1}{m} \sum_{\mu=1}^m \frac{\alpha \phi''(y_\mu)}{\alpha + \phi''(y_\mu) g_n(z)} d_\mu,
\end{align}
in which we defined 
\begin{align}\label{eq:def_dmu}
    d_\mu &\equiv \frac{g_n(z)}{z-x} - \frac{1}{z-x} \frac{\alpha + \phi''(y_\mu)g_n(z)}{\alpha + \phi''(y_\mu) \frac{1}{n} \bz_\mu^\intercal (M_n^{(\mu)} - z \mathrm{I}_n)^{-1} \bz_\mu} \frac{1}{n} \bz_\mu^\intercal (M_n^{(\mu)} - z \mathrm{I}_n)^{-1} \bz_\mu.
\end{align}
Let us denote $L_n(z,g) \equiv -g - \left[z-\frac{\alpha}{m} \sum_{\mu=1}^m \frac{\phi''(y_\mu)}{\alpha + \phi''(y_\mu) g}\right]^{-1}$. We know by \cite{silverstein1995empirical} that $g_{\alpha,\phi}[\nu_\by^m](z)$ is defined as the only solution in $\bbC_+$ to $L_n(z,g) = 0$. Let us first show, using eq.~\eqref{eq:cavity_gn}, that there exists $\eta \in (0,1)$ such that:
\begin{align}\label{eq:step_A}
    \lim_{n \to \infty} n^\eta \sup_{\by \in \bbR^m} |L_n(z,\EE g_n(z))| &= 0.
\end{align}
\begin{proofof}{eq.~\eqref{eq:step_A}}
First, notice that as $M_n^{(\mu)}$ is independent of the Gaussian vector $\bz_\mu$, we have, denoting 
$\{\lambda_i\}_{i=1}^n$ the eigenvalues of $(M_n^{(\mu)} - z \mathrm{I}_n)^{-1}$, that $|\lambda_i| \leq z_2^{-1}$, and moreover: 
\begin{align}
    \EE_\bz \left|\frac{1}{n}  \bz_\mu^\intercal (M_n^{(\mu)}-z\mathrm{I}_n)^{-1} \bz_\mu - g_n^{(\mu)}(z)\right|^2
    &= \EE_{\{\bz_\nu\}_{\nu (\neq \mu)}} \EE_{\bw} \left|\frac{1}{n}  \bw^\intercal (M_n^{(\mu)}-z\mathrm{I}_n)^{-1} \bw - g_n^{(\mu)}(z)\right|^2,
\end{align}
with $\bw \sim \mathcal{N}(0,\mathrm{I}_n)$ independently of all $\{\bz_\nu\}_{\nu (\neq \mu)}$.
Therefore, by rotational invariance of the Gaussian distribution, we reach:
\begin{align}
    \label{eq:bound_convergence_stieltjes}
    \EE_\bz \left|\frac{1}{n}  \bz_\mu^\intercal (M_n^{(\mu)}-z\mathrm{I}_n)^{-1} \bz_\mu - g_n^{(\mu)}(z)\right|^2
    \nonumber
    &= \EE_{\{\bz_\nu\}_{\nu (\neq \mu)}} \EE_\bw \Bigg[\Big\{\frac{1}{n}  \sum_{i=1}^n \lambda_i (w_i^2 - 1) \Big\}^2\Bigg], \\ 
    \nonumber
    &= \EE_{\{\bz_\nu\}_{\nu (\neq \mu)}} \EE_\bw \Big[\frac{1}{n^2}  \sum_{i,j} \lambda_i \lambda_j (w_i^2 - 1) (w_j^2 - 1) \Big], \\ 
    \nonumber
    &= \EE_{\{\bz_\nu\}_{\nu (\neq \mu)}} \EE_\bw \Big[\frac{1}{n^2}  \sum_{i} \lambda_i^2 (w_i^2 - 1)^2\Big], \\ 
    &\leq \frac{1}{z_2^2} \frac{1}{n^2} \EE_{\bw} \sum_{i=1}^n (w_i^2 - 1)^2 \leq \frac{K}{z_2^2 n},
\end{align}
with a $K> 0$ (independent of $\by$ and z).
Moreover, let us denote $\Delta \equiv \frac{z_2}{2 (C^2 + z_1^2) + z_2^2}$, in which $C$ is the constant of Lemma~\ref{lemma:log_pot_existence}. Thanks to this lemma, we have that (uniformly in $\by$ and a.s.) $\liminf_{n \to \infty} \mathrm{Im}[g_n(z)] \geq \Delta$. In particular, using this bound, we can easily show (cf Part~4 of \cite{silverstein1995empirical} for a completely similar argument) that, uniformly in $\by$ and almost surely:
\begin{subnumcases}{\label{eq:bound_lemma_silverstein_2}}
    \max_{1 \leq \mu \leq m} \left|\frac{\alpha \phi''(y_\mu)}{\alpha+\phi''(y_\mu) g_n(z)}\right| \leq \frac{\alpha}{\Delta}, &\\
    \lim_{n \to \infty} \left[\max_{1 \leq \mu \leq m} \max \left\{\left|\frac{\alpha + \phi''(y_\mu)g_n(z)}{\alpha + \phi''(y_\mu) \frac{1}{n} \bz_\mu^\intercal (M_n^{(\mu)} - z \mathrm{I}_n)^{-1} \bz_\mu} - 1\right|, |x - x^{(\mu)}|\right\}\right] = 0.&
\end{subnumcases}
Combining eqs.~\eqref{eq:bound_convergence_stieltjes},\eqref{eq:bound_lemma_silverstein_2} into eqs.~\eqref{eq:cavity_gn},\eqref{eq:def_dmu} yields that there exists $\eta \in (0,1)$ such that eq.~\eqref{eq:step_A} is satisfied.
\end{proofof}
\\
By eq.~\eqref{eq:step_A} there exists $K > 0$ and a function $K(\by) \in \bbC$, such that for all $\by$: $|K(\by)| \leq K$, and
\begin{align}\label{eq:implicit_eg_n}
    \EE g_n(z) &= - \left[z - \frac{\alpha}{m} \sum_{\mu=1}^m \frac{\phi''(y_\mu)}{\alpha + \phi''(y_\mu) \EE g_n(z)}\right]^{-1} + \frac{K(\by)}{n^\eta}.
\end{align}
We write $g_{\alpha,\phi}[\nu_\by^m](z) = m_1(z) + i m_2(z)$ and $\EE g_n(z) = m_{n,1}(z) + i m_{n,2}(z)$. Note that all imaginary parts are strictly positive. Taking the imaginary part of eq.~\eqref{eq:implicit_eg_n}:
\begin{align}\label{eq:imaginary_implicit_eg_n}
    m_{n,2}(z) &= \frac{\mathrm{Im}(K(\by))}{n^\eta} + \frac{z_2 + \frac{\alpha}{m} \sum_{\mu=1}^m \frac{\phi''(y_\mu)^2 m_{n,2}(z)}{\left|\alpha + \phi''(y_\mu) \EE g_n(z)\right|^2}}{\left|z - \frac{\alpha}{m} \sum_{\mu=1}^m \frac{\phi''(y_\mu)}{\alpha + \phi''(y_\mu) \EE g_n(z)}\right|^{2}}.
\end{align}
Using eq.~\eqref{eq:implicit_eg_n} and its counterpart for $g_{\alpha,\phi}[\nu_\by^m](z)$ (which does not have a second term), we reach
\begin{align}\label{eq:contraction_Egn}
    g_{\alpha,\phi}[\nu_\by^m](z)  - \EE g_n(z) &= \frac{K(\by)}{n^\eta} + \left[g_{\alpha,\phi}[\nu_\by^m](z)  - \EE g_n(z)\right] A_n(z), 
\end{align}
with 
\begin{align}
    A_n(z) &\equiv \frac{\frac{\alpha}{m}\sum_{\mu} \frac{\phi''(y_\mu)^2}{\left(\alpha + \phi''(y_\mu) \EE g_n(z)\right)\left(\alpha + \phi''(y_\mu) g_{\alpha,\phi}[\nu_\by^m](z)\right)}}{\left[z - \frac{\alpha}{m} \sum_{\mu=1}^m \frac{\phi''(y_\mu)}{\alpha + \phi''(y_\mu) \EE g_n(z)}\right]\left[z - \frac{\alpha}{m} \sum_{\mu=1}^m \frac{\phi''(y_\mu)}{\alpha + \phi''(y_\mu) g_{\alpha,\phi}[\nu_\by^m](z)}\right]}.
\end{align}
By the Cauchy-Schwarz inequality, $|A_n(z)| \leq \sqrt{A_1(z) A_2(z)}$, with 
\begin{subnumcases}{}
    A_1(z) \equiv \frac{\alpha}{m} \sum_{\mu=1}^m \frac{\phi''(y_\mu)^2}{|\alpha + \phi''(y_\mu) \EE g_n(z)|^2} \left|z - \frac{\alpha}{m} \sum_{\mu=1}^m \frac{\phi''(y_\mu)}{\alpha + \phi''(y_\mu) \EE g_n(z)}\right|^{-2}, & \\ 
    A_2(z) \equiv \frac{\alpha}{m} \sum_{\mu=1}^m \frac{\phi''(y_\mu)^2}{|\alpha + \phi''(y_\mu) g_{\alpha,\phi}[\nu_\by^m](z)|^2} \left|z - \frac{\alpha}{m} \sum_{\mu=1}^m \frac{\phi''(y_\mu)}{\alpha + \phi''(y_\mu) g_{\alpha,\phi}[\nu_\by^m](z)}\right|^{-2}. &
\end{subnumcases}
In particular, using the counterpart of eq.~\eqref{eq:imaginary_implicit_eg_n} for $m_2(z)$, we have:
\begin{align}\label{eq:A1_bound}
    A_1(z) &= \frac{\alpha}{m} \sum_{\mu=1}^m \frac{\phi''(y_\mu)^2 m_2(z)}{|\alpha + \phi''(y_\mu) \EE g_n(z)|^2} \frac{1}{z_2 + \frac{\alpha}{m} \sum_{\mu=1}^m \frac{\phi''(y_\mu)^2 m_{n,2}(z)}{\left|\alpha + \phi''(y_\mu) \EE g_n(z)\right|^2}}.
\end{align}
Using items $(i)$ and $(ii)$ of Lemma~\ref{lemma:log_pot_existence}, we have the inequalities: $\frac{z_2}{2(C^2+z_1^2)+z_2^2} \leq m_2(z) \leq \frac{1}{z_2}$ and $\frac{z_2}{2(C^2+z_1^2)+z_2^2} \leq m_{n,2}(z) \leq \frac{1}{z_2}$. This implies from eq.~\eqref{eq:A1_bound}:
\begin{align}
    A_1(z) &\leq \frac{z_2}{2(C^2+z_1^2)+z_2^2} \frac{1}{z_2 + \frac{z_2}{2(C^2+z_1^2)+z_2^2}}.
\end{align}
In particular, there exists a constant $\Gamma \in (0,1)$ such that $A_1(z) \leq \Gamma$ uniformly in $n,\by$. Similarly, we find that $A_2(z) \leq \Gamma (1 + \frac{K}{z_2 n^\eta})$ uniformly in $\by$. So we have from eq.~\eqref{eq:contraction_Egn}:
\begin{align}
    \limsup_{n \to \infty}\left\{n^\eta \sup_{\by \in \bbR^m} \left|\EE g_n^D(z) - g_{\alpha,\phi}[\nu_\by^m](z)\right| \right\} &\leq \frac{K}{1-\Gamma} < \infty, 
\end{align}
which proves item $(ii)$.
\end{proof}
\section{Technical aspects of the quenched calculation}\label{sec:app_quenched}

 \subsection{The phase volume factor}\label{subsec_app:quenched_phasevolume}
 
Introducing the Fourier transform of the delta functions, we reach at leading exponential order in $n$:
\begin{align*}
   \frac{1}{n} \ln \left[\prod_{a=1}^p \int_{\bbR^n} \mathrm{d}\bx^a \prod_{a \leq b} \delta\left(n q_{ab} - n  \bx^a \cdot \bx^b\right)\right] &\simeq \frac{p}{2} \ln \frac{2\pi}{n} + \frac{1}{2} \sup_{\{\hat{q}_{ab}\}}\left[\sum_{a,b} q_{ab} \hat{q}_{ab} - \ln \det \hat{\bq}\right].
\end{align*}
The replica symmetric assumption can be made on the variables $\hat{\bq}$ that achieve this supremum : $\hat{q}_{aa} = \hat{q}_0$ and $\hat{q}_{ab} = -\hat{q}$ for $a \neq b$. This leads to $\det \hat{\bq} = (\hat{q}_0 + \hat{q})^{p-1} (\hat{q}_0 - (p-1)\hat{q})$, and after taking the $p \to 0^+$ limit, we reach:
{\small 
\begin{align*}
   &\frac{1}{np} \ln \left[\prod_{a=1}^p \int_{\bbR^n} \mathrm{d}\bx^a \prod_{a \leq b} \delta\left(n q_{ab} - n  \bx^a \cdot \bx^b\right)\right] \simeq \frac{1}{2} \log \frac{2 \pi}{n}+ \frac{1}{2} \sup_{\hat{q}_0,\hat{q}} \left[\hat{q}_0 + q \hat{q} - \log(\hat{q}_0 + \hat{q}) + \frac{\hat{q}}{\hat{q}_0 + \hat{q}} \right].
\end{align*}
}
The diverging term $-\frac{1}{2} \log n$ will be canceled out by the joint density of the gradients as we will see later.
The solution of the supremum is easy to carry out, and we finally reach eq.~\eqref{eq:result_phaseterm_replicas}.

\subsection{The joint density of the gradients}\label{subsec_app:quenched_gradient_density}

We denote $S = \mathrm{Span}\, (\{\bx^a\}_{a=1}^p)$.
Following \cite{ros2019complex}, for every $1 \leq a \leq p$ we can construct an orthonormal basis of $S$, denoted $(\bbe^a_{b})_{1 \leq b \leq p}$ for which $\bx^a$ is the first vector, that is $\bbe^a_a = \bx^a$. This basis is convenient, since $\{\bx^a\}^\perp \cap S = \mathrm{Span} \, (\{\bbe^a_b\}_{b (\neq a)})$.
We can also chose an arbitrary orthonormal basis $(\bbe_{p+1}, \cdots, \bbe_{n})$ of $S^\perp$. With this choice of basis, 
we can see that the gradient $\mathrm{grad} \, L(\bx^a)$ is identified with the vector in $\bbR^{n-1}$ with components:
\begin{align}\label{eq:gradient_projection}
   \mathrm{grad} \, L(\bx^a) &
   = \left(
     \{\nabla L(\bx^a) \cdot \bbe^a_i\}_{i=1}^{a-1}, 
     \{\nabla L(\bx^a) \cdot \bbe^a_i\}_{i=a+1}^{p}, 
     \{\nabla L(\bx^a) \cdot \bbe_i\}_{i=p+1}^{n}
    \right).
\end{align}
Recall that $\nabla L(\bx^a) = \frac{1}{m} \sum_{\mu} \bxi_\mu \, \phi'(y^a_\mu)$.
Let us make a few remarks:
\begin{itemize}
   \item For every $a$, the basis $(\bbe^a_b)_{b=1}^p$ is only a function of the values of the overlaps $\{q_{ab}\}$.
   \item We consider the joint density of the gradients \emph{conditioned} by the value of $\{\by^a\}$.
   In particular, this means that for every $a \neq b$, $\nabla L(\bx^a) \cdot \bbe^a_b$ is fixed by the values of $\{\by^c\}_{c=1}^p$
   and the overlaps $q_{ab}$. In particular, the first $(p-1)$ components of eq.~\eqref{eq:gradient_projection} are deterministic, thus
   their density will yield delta functions that are constraints on $\{\by^a\}$ and $\{q_{ab}\}$.
   \item The last $n-p$ components of eq.~\eqref{eq:gradient_projection} are (at fixed $\{\by^a\}$) zero mean Gaussian random variables with covariance given by $\EE \left[\mathrm{grad} L(\bx^a)_i \, \mathrm{grad} L(\bx^b)_j\right] = \frac{\delta_{ij}}{m^2} \sum_{\mu} \phi'(y^a_\mu) \, \phi'(y^b_\mu)$.
   Their joint density taken at $0$ is thus at leading exponential order in the $n \to \infty$ limit:
   \begin{align}\label{eq:gradient_density_1}
     \exp\left\{\frac{np}{2} \log \frac{m}{2\pi} - \frac{n}{2} \log \det \left[\left(\frac{1}{m} \sum_{\mu=1}^m \phi'(y^a_\mu) \phi'(y^b_\mu)\right)_{1 \leq a,b \leq p}\right] \right\}.
   \end{align}
\end{itemize}
Given these remarks and eq.~\eqref{eq:gradient_density_1}, in order to complete the calculation of the joint gradient density we
need to compute the quantities $(\nabla \,L(\bx^a) \cdot \bbe^a_b)$ for every $a \neq b$ as a function 
of $\{y^a_\mu\}$ and $\{q_{ab}\}$. In order to simplify the calculation, we will already make use of the replica-symmetric assumption on $q$, 
that is we assume $q_{aa} = 1$ and $q_{ab} = q$ for $a\neq b$.
Let us now describe a possible construction for the basis $(\bbe^a_b)_{b=1}^p$. We introduce three auxiliary quantities that are functions 
of $q$ and $p$:
\begin{subnumcases}{\label{eq:def_auxiliary_functions}}
 f^0_{p}(q) \equiv \frac{1}{p-1} \left[\frac{p-2}{\sqrt{1-q}} + \frac{1}{\sqrt{1+(p-2)q - (p-1)q^2}}\right], &\\
  f_{p}(q) \equiv \frac{1}{p-1} \left[-\frac{1}{\sqrt{1-q}} + \frac{1}{\sqrt{1+(p-2)q - (p-1)q^2}}\right]  , &\\
  z_{p}(q) \equiv - \frac{q}{\sqrt{1+(p-2)q - (p-1)q^2}}.&
\end{subnumcases}
With these definitions, we can consider:
\begin{align}\label{eq:basis}
    \begin{cases} 
      \bbe^a_a &\equiv\bx^a, \\
  \bbe^a_b &\equiv z_p(q) \bx^a + f^0_p(q) \bx^b + f_p(q) \sum_{c (\neq a,b)} \bx^c, \qquad (b \neq a). 
    \end{cases}
\end{align}
It is straightforward to check from eq.~\eqref{eq:basis} that we have for all $a,b,c$ that $\bbe^a_b \cdot \bbe^a_c = \delta_{bc}$.
We can now see that the delta term of the joint density of the gradients taken at $0$ is:
\begin{align}
   \label{eq:gradient_density_2}
   \prod_{a \neq b} \delta \left[\nabla L_1(\bx^a) \cdot \bbe^a_b\right] &=\prod_{a \neq b} \delta \left[\frac{1}{m} \sum_{\mu=1}^m \phi'(y^a_\mu) \left(z_p(q) y^a_\mu + f_p^0(q) y^b_\mu + f_p(q)\sum_{c (\neq a,b)} y^c_\mu \right)\right].
\end{align}
The product of eq.~\eqref{eq:gradient_density_1} and eq.~\eqref{eq:gradient_density_2} gives eq.~\eqref{eq:gradient_density}.

\subsection{Decoupling the replicas and taking the \texorpdfstring{$p \to 0^+$}{p->0} limit}\label{subsec_app:quenched_decoupling}

\subsubsection{Replica symmetry and decoupling}

In order to apply the replica method, we need to be able to take the $p \to 0$ limit, by analytically extending eq.~\eqref{eq:p_moment} to all $p > 0$.
The main idea is that we expect replica symmetry to influence the measure $\nu$ that appears in eq.~\eqref{eq:p_moment}. 
More precisely, we expect that for all permutation $\pi \in \mathcal{S}_p$, we have $\nu(\mathrm{d} \lambda^1, \cdots, \mathrm{d}\lambda^p) = \nu(\mathrm{d}\lambda^{\pi(1)}, \cdots, \mathrm{d}\lambda^{\pi(p)})$. 
We make in eq.~\eqref{eq:p_moment} the substitution:
\begin{align}
    \sup_{\nu \in \mathcal{M}(p,q)} \to \sup_{\{\mu_a\}_{a=1}^p \in {\cal M}(\bbR)} \, \, \sup_{\substack{\nu \in {\cal M}(p,q)\\\mathrm{s.t.}\, \{\nu^a = \mu_a\}}}
\end{align}
In this last expression, the replica symmetric assumption leads us in particular to assume that $\mu_a = \mu$ for all $a$. In order to make the remaining calculation tractable we will also need to fix some linear statistics of $\nu$ via Lagrange multipliers:
\begin{itemize}
   \item For every $a \leq b$, we will fix the linear statistics $
     \int \nu(\mathrm{d}\lambda) \, \phi'(\lambda^a) \phi'(\lambda^b) = A_{ab}$,
   with Lagrange multipliers $\hat{A}_{ab}$. Note that by replica symmetry, we actually assume that $A_{ab} = a$ for $a \neq b$ and $A_{aa} = A$ (and samely for the Lagrange multipliers).
   \item For all $a,b$ we will fix the linear statistics $\int  \nu(\mathrm{d}\lambda) \,\phi'(\lambda^a) \lambda^b = B_{ab}$,
   with Lagrange multipliers $\hat{B}_{ab}$. By replica symmetry, we assume that $B_{aa} = B$ and $B_{ab} = b$.
\end{itemize}
Combining these remarks, we reach that the $\nu$-dependent term of eq.~\eqref{eq:p_moment} is equal to:
\begin{align}\label{eq:first_simplification_replicas}
  &\sup_{\mu \in \mathcal{M}_\phi(B)} \sup_{\substack{A,a \\ B,b}} \extr_{\substack{\hat{A},\hat{a} \\ \hat{B}, \hat{b}} } \sup_{\substack{\nu \in \mathcal{M}(\bbR^n) \\ \mathrm{s.t.}\, \{\nu^a = \mu\}}} \Bigg[p \kappa_{\alpha,\phi}\left[\mu,t_\phi(\mu)\right] - \frac{1}{2} \ln \det \left[\{A_{ab}\}\right] - \sum_{a,b} \left[\frac{1}{2} A_{ab} \hat{A}_{ab} + B_{ab} \hat{B}_{ab}\right]  \nonumber \\
  &+ \sum_{a,b} \left[\frac{1}{2} \hat{A}_{ab} \int \nu(\mathrm{d}\lambda) \phi'(\lambda^a) \phi'(\lambda^b) + \hat{B}_{ab} \int \nu(\mathrm{d}\lambda) \phi'(\lambda^a) \lambda^b \right]  - \alpha H(\nu|\mu_{\mathrm{G},q}) \Bigg].
\end{align}
Note that here we did not always explicit the replica-symmetry assumption on all the variables to obtain more compact expressions.
The supremum over $B,b$ is moreover constrained by the following condition of eq.~\eqref{eq:conditions_Mpq_2}: $\forall a \neq b, \ z_p(q) B_{aa} + f^0_p(q) B_{ab} + f_p(q) \sum_{c (\neq a,b)} B_{ac} = 0$.  
Under the replica symmetric assumption, this becomes:
\begin{align}\label{eq:conditions_B}
  z_p(q) B + f^0_p(q) b + f_p(q) (p-2) b &= 0. 
\end{align}
Again, we introduce Lagrange multipliers $C_{ab}$ to fix these conditions, that reduce to $C_{ab} = C$ because of replica symmetry.
Finally, in order to fix the marginal distributions of $\nu$, we will have to introduce 'functional' Lagrange multipliers $g^a(\lambda^a)$. 
Again, by replica symmetry, we expect all of them to be equal to $g(\lambda^a)$. In the end, we reach the simplification of eq.~\eqref{eq:first_simplification_replicas}:
\begin{align}\label{eq:second_simplification_replicas}
  &\sup_{\substack{\mu \in \mathcal{M}_\phi(B) \\ \nu \in \mathcal{M}(\bbR^n)}} \sup_{\substack{A,a \\ B,b}} \extr_{\substack{C, \hat{A},\hat{a} \\ \hat{B}, \hat{b}, \{g(\lambda)\} } } \Bigg\{p \kappa_{\alpha,\phi}\left[\mu,t_\phi(\mu)\right] - \frac{1}{2} \ln \det \left[\{A_{ab}\}\right] - \sum_{a,b} \left[\frac{1}{2}\hat{A}_{ab} A_{ab} + \hat{B}_{ab} B_{ab} \right]  \\
  &  - p \int \mu(\mathrm{d} \lambda) g(\lambda) + p(p-1)C\left[B z_p(q) + b \left(f^0_p(q) +(p-2) f_p(q) \right) \right]  - \alpha H(\nu|\mu_{\mathrm{G},q}) \nonumber \\
  & + \sum_{a,b}\left[\frac{\hat{A}_{ab}}{2} \int \nu(\mathrm{d}\lambda) \phi'(\lambda^a) \phi'(\lambda^b) +  \hat{B}_{ab} \int \nu(\mathrm{d}\lambda) \phi'(\lambda^a) \lambda^b \right] + \sum_a \int \nu(\mathrm{d}\lambda) g(\lambda^a) \Bigg\}\nonumber.
\end{align}
We can now solve exactly the supremum over $\nu$. By a classical Gibbs measure calculation that we already detailed in section~\ref{sec:annealed}, we obtain (recall that $Q \in \bbR^{p \times p}$ is the overlap matrix):
\begin{align}\label{eq:third_simplification_replicas}
  &\sup_{\nu \in \mathcal{M}(\bbR^n) } \left\{\int \nu(\mathrm{d}\lambda) \left[ \sum_{a,b} \left( \frac{\hat{A}_{ab}}{2} \phi'(\lambda^a) \phi'(\lambda^b) + \hat{B}_{ab}  \phi'(\lambda^a) \lambda^b \right) + \sum_a g(\lambda^a) \right]- \alpha H(\nu|\mu_{\mathrm{G},q})\right\} \nonumber \\
  &= \alpha \ln \left[\int_{\bbR^p} \frac{\mathrm{d}\lambda}{\sqrt{2\pi}^p \sqrt{\det Q}} e^{\sum_{a,b} \left(-\frac{1}{2}  (Q^{-1})_{ab} \lambda^a \lambda^b + \frac{\hat{A}_{ab}}{2 \alpha} \phi'(\lambda^a) \phi'(\lambda^b) +  \frac{\hat{B}_{ab}}{\alpha} \phi'(\lambda^a) \lambda^b \right)+ \sum_a \frac{g(\lambda^a)}{\alpha}} \right].
\end{align}
To completely decouple the replicas, we will make use of two classical identities, for any $x,y$:
\begin{align*}
  e^{\frac{x^2}{2}} = \int \mathcal{D}\xi \, e^{\xi x}\ , \hspace{1cm}
  e^{xy} = \int \mathcal{D}\xi \mathcal{D}\xi' \, e^{\frac{x}{\sqrt{2}} (\xi + i \xi') + \frac{y}{\sqrt{2}} (\xi - i \xi')}.
\end{align*}
Thanks to replica symmetry, we can compute $Q^{-1}$ and $\det Q$ as:
\begin{subnumcases}{}
  \det Q = (1-q)^{p-1} [1 + (p-1)q] , &\\
  Q^{-1}_{ab} = \frac{1 +(p-1)q}{1 + (p-2)q -(p-1)q^2} \delta_{ab} - \frac{q}{(1-q)(1+(p-1)q)}. &
\end{subnumcases}
We define
\begin{subnumcases}{}
  d_{0,p}(q) \equiv  \frac{1 +(p-1)q}{1 + (p-2)q -(p-1)q^2}, &\\
  d_p(q) \equiv \frac{q}{(1-q)(1+(p-1)q)}.&
\end{subnumcases}
Using all the above, we can now simplify eq.~\eqref{eq:third_simplification_replicas}: 
\begin{align}
  \label{eq:fourth_simplification_replicas}
&\alpha \ln \left[\int_{\bbR^p} \frac{\mathrm{d}\lambda}{\sqrt{2\pi}^p \sqrt{\det Q}} e^{\sum_{a,b} \left(-\frac{1}{2}  (Q^{-1})_{ab} \lambda^a \lambda^b + \frac{\hat{A}_{ab}}{2\alpha} \phi'(\lambda^a) \phi'(\lambda^b) +  \frac{\hat{B}_{ab}}{\alpha} \phi'(\lambda^a) \lambda^b \right)+ \sum_a \frac{g(\lambda^a)}{\alpha}} \right] \nonumber \\
&= -\frac{\alpha p}{2} \ln 2 \pi - \frac{\alpha(p-1)}{2} \ln (1-q) - \frac{\alpha}{2} \ln [1 + (p-1)q] + \alpha \ln \left[\int_{\bbR^4} \mathcal{D} \bxi I_p(\bxi)^p \right],
\end{align}
in which we defined $\bxi \equiv (\xi_q,\xi_a,\xi_b,\xi_b')$ and
{\small
\begin{align*}
    I_p(\bxi) &\equiv \int \mathrm{d} \lambda \, e^{\frac{g(\lambda)}{\alpha} -\frac{d_{0,p}(q)\lambda^2 }{2} + \frac{\hat{A} - \hat{a}}{2\alpha} \phi'(\lambda)^2 + \frac{\hat{B} - \hat{b}}{\alpha} \phi'(\lambda) \lambda +  \sqrt{d_p(q)} \xi_q \lambda + \sqrt{\frac{\hat{a}}{\alpha}} \xi_a \phi'(\lambda) + \sqrt{\frac{\hat{b}}{2\alpha}} [\phi'(\lambda)(\xi_b + i \xi_b') + \lambda (\xi_b - i\xi_b') ] }.
\end{align*}
}%
While the involved expressions are very cumbersome, we have successfully decoupled the replicas.

\subsubsection{The \texorpdfstring{$p \to 0$}{p->0} limit and final result}

We begin by a remark on eq.~\eqref{eq:fourth_simplification_replicas}. Note that $\lim_{p \to 0} (1/p) \ln \left[\int \mathcal{D} \xi \, I_p(\bxi)^p \right] = \int \mathcal{D} \bxi \,\ln \left[I_0(\bxi) \right]$.
Thus, after multiplication by $(1/p)$, the $p \to 0^+$ limit of eq.~\eqref{eq:fourth_simplification_replicas} will yield:
\begin{align}\label{eq:first_term}
  &-\frac{\alpha}{2} \ln 2 \pi - \frac{\alpha}{2} \ln (1-q) - \frac{\alpha q}{2 (1-q)} + \alpha \int {\cal D}\bxi \ln I(\bxi),
\end{align}
in which $I(\bxi)$ is defined in Result~\ref{result:quenched_nosignal}. We can wrap up the calculation. We make two remarks. Firstly the condition eq.~\eqref{eq:conditions_B} reduces, in the $p \to 0$ limit, to $b = q B$, so that we will be able to simplify the terms involving the Lagrange multiplier $C$. Secondly, the variable $B$ is equal to $t_\phi(\mu)$, defined in Theorem~\ref{thm:annealed_nosignal}.
We combine now eqs.~\eqref{eq:conditions_Mpq_1},\eqref{eq:p_moment},\eqref{eq:second_simplification_replicas} and \eqref{eq:first_term} with the two remarks above. 
Changing notations from $\mu$ to $\nu$ and $B$ to $C$, we obtain finally the conclusion of Result~\ref{result:quenched_nosignal}.
\changelocaltocdepth{2}
\section{Generalization to more involved models}\label{sec_app:other_models}

\subsection{Annealed and quenched calculations for \texorpdfstring{$L_2$}{L2}}\label{subsec_app:L2}

We give here a sketch of the generalization of our annealed and quenched calculations to $L_2$, thus yielding Theorem~\ref{thm:annealed_signal} and Result~\ref{result:quenched_signal}.
A more detailed derivation of these results will be included in a future extended version of this work.
We restrict here to the annealed calculation (the generalization of the quenched calculation is completely similar). The majority of the arguments being identical to the $L_1$ case, we will only highlight the main differences and give the important intermediary results.

In the Kac-Rice formula, one has now to integrate over the overlap $q \equiv \bx \cdot \bx^\star$ as well. Moreover, we condition over the joint values of $a_\mu \equiv \bxi_\mu \cdot \bx$ and $b_\mu \equiv (1-q^2)^{-1/2}\left[(\bxi_\mu \cdot \bx^\star) - q a_\mu \right]$, rather than just $\bxi_\mu \cdot \bx$ (as we did for $L_1$). Note that $(a_\mu,b_\mu)$ follows a joint standard Gaussian distribution.Using these definitions, we can obtain the counterpart of Lemma~\ref{lemma:annealed_finite_n_L1} for $L_2$:
{\small
\begin{align*}
    \EE \, \mathrm{Crit}_{n,L_2}(B,Q) &= {\cal C}_n \int_{Q} \mathrm{d}q \ e^{\frac{n(1+\ln \alpha + \ln(1-q^2))}{2}} \EE_{\ba,\bb}\left[\delta(P_n(\ba,\bb))\mathds{1}_{L_2(\ba,\bb) \in B} e^{-n {\cal E}_n(\ba,\bb)} \EE_\bz \left|\det H_n(\ba,\bb)\right|\right],
\end{align*}
}
in which ${\cal C}_n$ is exponentially trivial, and we defined:
{\small
\begin{subnumcases}{}
P_n(\ba,\bb) \equiv \frac{1}{m} \sum_{\mu=1}^m b_\mu \phi'(a_\mu) \left[\phi\left(q a_\mu + \sqrt{1-q^2} b_\mu\right) - \phi(a_\mu)\right] , & \\
{\cal E}_n(\ba,\bb) \equiv \frac{1}{2} \ln \left\{\sum_{\mu=1}^m \phi'(a_\mu)^2 \left[\phi\left(q a_\mu + \sqrt{1-q^2} b_\mu\right) - \phi(a_\mu)\right]^2\right\}, & \\
H_n(\ba,\bb) \equiv \frac{1}{m} \sum_{\mu=1}^m \left[\phi'(a_\mu)^2 - \theta''(a_\mu) \left[\phi\left(\sqrt{1-q^2}b_\mu + qa_\mu\right)-\phi(a_\mu)\right]\right]\bz_\mu \bz_\mu^\intercal & \\
\hspace{2cm}- \left[\frac{1}{m} \sum_{\mu=1}^m a_\mu \phi'(a_\mu) \left[ \phi(a_\mu) - \phi\left(q a_\mu + \sqrt{1-q^2} b_\mu\right) \right]\right] \mathrm{I}_{n-2}. \nonumber& 
\end{subnumcases}
}
Here $\bz \in \bbR^{(n-2) \times m}$ is an i.i.d.\ standard Gaussian matrix. T
The condition $P_q(\ba,\bb) = 0$ arises from the conditioning on the nullity of the gradient in the linear subspace spanned by $\bx^\star$, and ${\cal E}_n(\ba,\bb)$ from the density of the gradient in the subspace orthogonal to $\{\bx,\bx^\star\}$.
A crucial feature of this equation is that the joint distribution of $(L_2(\bx),\mathrm{grad}L_2(\bx),\mathrm{Hess}L_2(\bx))$ only depends on $\bx$ via the overlap $q = \bx\cdot \bx^\star$ with the ``true'' solution. Once conditioned over the values of $q$, it thus becomes clear why the calculations made for $L_1$ will generalize here.

As in Section~\ref{subsec:large_deviations_annealed_L1}, one can then show the concentration of the empirical logarithmic potential on the functional $\kappa_{\alpha,\phi}(q,\nu_{\ba,\bb}^m)$, in which $\nu_{\ba,\bb}^m \in {\cal M}(\bbR^2)$ is now the empirical measure of $\{a_\mu,b_\mu\}_{\mu=1}^m$. 
We obtain the counterpart of Lemma~\ref{lemma:log_potential_L1}, under the analogue of Assumption~\ref{hyp:technical} for $H_n(\ba,\bb)$: there exists $\eta > 0$ such that for all $t > 0$:
\begin{align}
    \lim_{n \to \infty}\frac{1}{n^{1+\eta}} \ln \mathbb{P}\left[\left|\frac{1}{n} \ln \EE_\bz \left|\det H_n(\ba,\bb)\right| - \kappa_{\alpha,\phi}(q,\nu^m_{\ba,\bb})\right| \geq t\right] &= -\infty.
\end{align}
Thanks to this result, we perform then a saddle point on the overlap $q$ and the empirical measure $\nu \in {\cal M}(\bbR^2)$, using Sanov's theorem and Varadhan's lemma. This yields the result of Theorem~\ref{thm:annealed_signal}.

As a final note, there exists similar results to the one presented in Section~\ref{sec:variational} that allow to compute the density (and the logarithmic potential) of $\mu_{\alpha,\phi}[q,\nu]$,
via the computation of its Stieltjes transform.

\subsection{Generalizations to other models}\label{subsec_app:other_models}

Our calculations, both annealed and quenched, can be generalized straightforwardly to many other loss functions and models. As is clear for instance in the annealed computation of Section~\ref{sec:annealed}, the key features that must be present are a Gaussian distribution for the data, and a loss function $L(\bx)$ that only depends on the data samples $\bxi_\mu$ via their projection over a few vectors (as $\bx$ for $L_1(\bx)$ and $\bx, \bx^\star$ for $L_2(\bx)$). We give thereafter three examples of models, that can be found in \cite{engel2001statistical,mei2018landscape}, and for which the calculations can be performed.

\begin{model}[Binary linear classification]\label{model:binary_classification}
    Consider $n,m \geq 1$ such that $m/n \to \alpha > 1$.
    Let $\sigma : \bbR \to [0,1]$ a smooth threshold function. 
    We are given $m$ samples $(y_\mu,\bx_\mu)_{\mu=1}^m$ with $y_\mu \in \{0,1\}$ and $\bx_\mu \in \bbR^n$. 
    The elements of $(y_\mu)_{\mu=1}^m$ are generated according to $\mathbb{P}\left(Y_\mu = 1 | \bX_\mu = \bx\right) = \sigma(\bm{\theta}_0\cdot \bx)$, and the $\bx_\mu$ are i.i.d.\ standard Gaussian random variables in $\bbR^n$. 
    We want to learn the vector $\bm{\theta}_0 \in \mathbb{S}^{n-1}$ by minimizing the loss function:
    \begin{align}
        L(\bm{\theta}) &\equiv \frac{1}{2m} \sum_{\mu=1}^m \left[y_\mu - \sigma \left(\bm{\theta} \cdot \bx_\mu\right)\right]^2, \hspace{1cm} \bm{\theta} \in \mathbb{S}^{n-1}. 
    \end{align}
\end{model}

\begin{model}[Mixture of two Gaussians]\label{model:mixture_gaussians}
    Consider $n,m \geq 1$ such that $m/n \to \alpha > 1$. We are given $m$ samples $\by_\mu \in \bbR^n$, generated as $\by_\mu \overset{\mathrm{i.i.d.}}{\sim} \sum_{a=1}^2 p_a {\cal N}(\bm{\theta}^0_a, \mathrm{I}_n)$. The proportions $p_1,p_2$ are known, and we wish to recover $\bm{\theta}^0_1$ and $\bm{\theta}^0_2$ by minimizing the maximum-likelihood estimator:
    \begin{align}
        L(\bm{\theta}_1,\bm{\theta}_2) &\equiv -\frac{1}{m} \sum_{\mu=1}^m \ln \left[\frac{1}{2} \sum_{a=1,2} \frac{1}{\sqrt{2\pi}^n} \exp \left\{-\frac{1}{2} \norm{\by_\mu - \bm{\theta}_a}^2\right\}\right].
    \end{align}
    \end{model}
\begin{model}[Simple unsupervised learning model]\label{model:unsupervised_learning}
    Consider $n,m \geq 1$ such that $m/n \to \alpha > 1$. Let $\phi : \bbR \to \bbR$ a smooth activation function, $V : \bbR \to \bbR_+$ a ``potential'', and $\bx^0 \in \mathbb{S}^{n-1}$ a fixed vector. We assume that we are given i.i.d.\ data samples $\{\bxi_\mu\}_{\mu=1}^m \in \bbR^n$ distributed such that their projection on $\bx^0$ has a probability density $P(\bxi_\mu \cdot \bx = h) \propto e^{-\frac{1}{2} h^2 - V(h)}$, and the other coordinates of $\bxi_\mu$ are i.i.d.\ standard Gaussian variables.
    We wish to recover the vector $\bx^0$ by minimizing:
    \begin{align}
        L(\bx) &\equiv \frac{1}{m} \sum_{\mu=1}^m \phi\left(\bxi_\mu \cdot \bx\right), \hspace{1cm} \bx \in \mathbb{S}^{n-1}. 
    \end{align}
\end{model}
For each of these three models one can replicate the annealed and quenched calculations of Sections~\ref{sec:annealed} and \ref{sec:quenched}, under suitable technical hypotheses. 

\paragraph{A note on non-spherical priors}
Finally, it is clear from the calculation of Appendix~\ref{sec:app_quenched} 
(particularly Section~\ref{subsec_app:quenched_phasevolume}) 
that we can also generalize these techniques (at least heuristically) to non-spherical prior distributions on the vectors $\bx$.
The most natural hypothesis that allows the computation to be generalized is that the prior distribution takes the decoupled form $P(\mathrm{d}\bx) = \prod_{i} P(\mathrm{d}x_i)$.

\end{document}